\documentclass[10pt,journal,compsoc]{IEEEtran}
\usepackage{times}
\usepackage{xcolor}
\usepackage{soul}
\usepackage[utf8]{inputenc}
\usepackage[small]{caption}

\usepackage{amsmath,amssymb}
\usepackage{booktabs, makecell}
\usepackage{pifont}
\usepackage{dsfont}
\usepackage{graphicx}
\usepackage{subcaption}
\usepackage{diagbox}
\usepackage{enumitem}
\usepackage{mdframed}
\usepackage{amsthm}
\usepackage{color}
\usepackage{multirow}
\usepackage{comment}

\newtheorem{theorem}{\textbf{Theorem}}
\newtheorem{lemma}{\textbf{Lemma}}

\newtheorem{definition}{\textbf{Definition}}

\usepackage{tikz}
\usepackage{pgfplots}
\pgfplotsset{compat=1.5}
\usetikzlibrary{calc}
\usepackage{tkz-graph}
\usepackage{tkz-graph}
\GraphInit[vstyle =Normal]
\tikzset{
	LabelStyle/.style = { rectangle, rounded corners, draw,
		minimum width = 2em, fill = white!50,
		text = black, font = \bfseries },
	VertexStyle/.append style = { inner sep=1pt,
		font = \Large\bfseries},
	EdgeStyle/.append style = {->, ultra thick,gray!50} }
\usetikzlibrary{shadings,intersections,patterns}
\usepackage[linesnumbered,ruled,vlined]{algorithm2e}
\usepackage{algpseudocode}
\usetikzlibrary{pgfplots.groupplots}
\usepackage{slashbox}
\usepackage{color, colortbl}
\ifCLASSINFOpdf
\else
\fi
\begin{document}
%
\title{Real-Time Network-Level Traffic Signal Control: An Explicit Multiagent Coordination Method}
%
%
%

\author{Wanyuan ~Wang,
        Tianchi~Qiao,Jinming~Ma,Jiahui~Jin,Zhibin~Li,Weiwei~Wu,~and~Yichuan Jiang
\thanks{W. Wang, T. Qiao, J. Jin, W. Wu, and Y. Jiang are with the School of
Computer Science and Engineering, Southeast University, Nanjing, China
(e-mail: wywang@seu.edu.cn; tianchi-qiao@seu.edu.cn; jjin@seu.edu.cn;  weiweiwu@seu.edu.cn; yjiang@seu.edu.cn).}
\thanks{Z. Li is with the School of Transportation, Southeast University, Nanjing, China (e-mail:lizhibin@seu.edu.cn).}
\thanks{J. Ma is with the School of Computer Science and Technology, University of Science and Technology of China, Hefei, China (e-mail:jinmingm@mail.ustc.edu.cn).}
\thanks{Manuscript received Sep. 25, 2022.}
}


%
%

\markboth{IEEE TRANSACTIONS ON Mobile Computing,~VOL.~1, No.~8, September~2022}%
{Shell \MakeLowercase{\textit{Wang et al.}}: Real-Time Network-Level Traffic Signal Control: An Explicit Multiagent Coordination Approach}
%



\IEEEtitleabstractindextext{%
	\begin{abstract}
		Efficient traffic signal control (TSC) has been one of the most useful ways for reducing urban road congestion. Key to the challenge of TSC includes 1) the essential of real-time signal decision, 2) the complexity in traffic dynamics, and 3) the network-level coordination. Recent efforts that applied reinforcement learning (RL) methods can query policies by mapping the traffic state to the signal decision in real-time, however, is inadequate for unexpected traffic flows. By observing real traffic information, online planning methods can compute the signal decisions in a responsive manner. Unfortunately, existing online planning methods such as model-predict control and schedule driven heuristics, either require high computation complexity or get stuck in local coordination, and suffer from scalability and efficiency issues. Against this background, we propose an explicit multiagent coordination (EMC)-based online planning methods that can satisfy adaptive, real-time and network-level TSC. By multiagent, we model each intersection as an autonomous agent, and the coordination efficiency is modeled by a cost (i.e., congestion index) function between neighbor intersections. By network-level coordination, each agent exchanges messages with respect to cost function with its neighbors in a fully decentralized manner. By real-time, the message passing procedure can interrupt at any time when the real time limit is reached and agents select the optimal signal decisions according to the current message. Moreover, we prove our EMC method can guarantee network stability by borrowing ideas from transportation domain. Finally, we test our EMC method in both synthetic and real road network datasets. Experimental results are encouraging: compared to RL and conventional transportation baselines, our EMC method performs reasonably well in terms of adapting to real-time traffic dynamics, minimizing vehicle travel time and scalability to city-scale road networks.
	\end{abstract}
	
	}

\maketitle


%
\IEEEpeerreviewmaketitle

\section{Introduction}
%
%
%
%
\IEEEPARstart{N}owadays, traffic signal control (TSC) is still dominated by the use of fixed timing rules such as SCATS \cite{Lowrie1990SCATSSC} and SCOOT \cite{Robertson91}, which can be pre-defined by domain experts. However, with the ever-increasing number of vehicles traveling on urban road networks, these pre-defined TSC rules have become incomplete and outdated, thereby traffic congestion are becoming the key inconvenience to drivers. According to the INRIX Global Traffic Scorecard \cite{Trafficstatics}, in 2019, an average driver in the U.S. lost 99 hours a year due to congestion, costing them nearly \$88 billion. Intelligent TSC has the potential to reduce traffic congestion, energy consumption and CO$_{2}$ emissions \cite{VallatiMSCM16,8824092,Iwase0GC22}. It is commonly recognized that modern intelligent TSC should be real-time responsive to dynamic traffic flows and can scale-up to network-level scenarios \cite{SmithBXR13}.

To dynamically adjust TSC decisions according to real traffic, reinforcement learning (RL) methods can query the learned policy by mapping real traffic information to the signal decisions \cite{iet-its.2017.0153,WeiZYL18,ZhangYZ20,ChenWXZYXXL20,WeiZGL20}. Multiagent RL (MARL) focuses on coordinating traffic signals, where the individual policy can be regulated by the global joint state-action critics \cite{Chu2020,Ma2020}, and temporal and spatial influences can be learned by deep networks \cite{WeiXZZZC0ZXL19,YuLWJHC0020,JiangQS0Z22,WangXNTCX22}. To scale-up, hierarchical RL can be applied by decomposing the large-scale road networks into sub-networks, each sub-network controls its signals in the low-level, and the centralized upper-level constrains and evaluates the sub-network policy \cite{Pol2016,Tan2020,ZhangYZ20,XuWWJL21}. Unfortunately, RL based approaches have a couple of issues that have prevented deployment on real-world TSC: 1) the deep network-based implicit coordination requires careful implementation and tedious parameter-tuning, and 2) MARL is trained offline for regular traffic flows, but is not able to deal with unexpected situations such as emergency events and disruptions \cite{SmithBXR13,7782850,HuS19}.


\begin{table*}[t!]
	\centering
	\begin{tabular}{c|c|c|c|c|c}
		\hline\hline
		\multicolumn{2}{c|} {\multirow{2}* {\backslashbox{Category}{Property}}}  & Response Time  & Coordination  & Parameter & Scalability \\
		\multicolumn{2}{c|}{} & (seconds)  &Mechanism & (Based/Free) & {\# of Intersections}  \\
		\hline\hline
			\multirow{3}* {Offline RL} & Independent \cite{Pol2016,WeiZYL18,Hua2019,ChenWXZYXXL20} & \multirow {3}* {$\leq$3}  & None & \multirow {3}* {Based} & $\sim$ 2500 \\
			\cline{2-2}\cline{4-4}\cline{6-6}
		 & Coordinated \cite{KuyerWBV08,El-TantawyAA13,WeiXZZZC0ZXL19,YuLWJHC0020,WangXNTCX22} &  & Local &  & $\leq$ 200 \\
		 \cline{2-2}\cline{4-4}\cline{6-6}
		& Hierarchical \cite{Tan2020,ZhangYZ20,Ma2020} &  & Global &  & $\leq$36 \\
		\hline \hline
		
		\multirow{5}* {Online Planning}
		& Max Pressure \cite{VARAIYA2013177} & $\leq$3 & None & Free & $\sim$2500 \\
		\cline{2-6}
		& Schedule-Driven \cite{XieSB12,SmithBXR13,HuS17,GoldsteinS18,HuS19,Wuna2020} & $\leq$3 & Global & Based & $\leq$25 \\
		\cline{2-6}
		& Centralized MPC \cite{Cai09,7637098,ABOUDOLAS2010680,LinSXH11} & $\geq$3 & Global & Free & $\leq$15 \\
		\cline{2-6}
		& Decentralized MPC \cite{DEOLIVEIRA2010,Zhou17} & $\geq$80 & Global & Free & $\leq$34 \\
		\cline{2-6}
		& Our EMC Approach & $\leq$3 & Global & Free & $\sim$ 400 \\
		\hline\hline
	\end{tabular}
	\caption{Summary of the related TSC methods on response time, coordination mechanism, parameters-free/based, and the scalability. Parameter-based: in offline RL, a set of hyper-parameters are necessary for tuning deep networks and schedule-driven is sensitive to the fine-grained arrival time of each vehicle.}\label{Tab:RelatedWork}
\end{table*} 

Online planning methods attempt to search the signal decision for current traffic flows in real-time and is capable of responsiveness for irregular traffic flows. For example, by building a traffic prediction model, model-predictive control (MPC) can plan a sequence of signal actions for the real traffic \cite{8707103}. Unfortunately, MPC needs a complex mathematical program to optimize multiple intersections' signal sequences \cite{ABOUDOLAS2010680,LinSXH11,Zhou17,7328307,DEOLIVEIRA2010,Wuna2020}, which is inherently susceptible to scalability issues. A recent development in decentralized online planning that overcomes this scalability issue is schedule-driven methods \cite{SmithBXR13,AhmadMY17}. The key idea behind this method is to formulate TSC problem at each intersection as a machine scheduling problem, where jobs are represented by clusters of spatially adjacent vehicles \cite{XieSB12}. Scalability is ensured by allowing each intersection to use the polynomial schedule algorithm to search for its own signal plan \cite{HuS17,GoldsteinS18}. Coordination is achieved by communicating outflows from upstream intersections to downstream intersections \cite{HuS19}. The success of the schedule-driven approach depends on the fine-grained sensing information (i.e., the arrival time) of each vehicle, which is hard to achieve and error-prone in real-world complex traffic dynamics \cite{WeiZGL20}.

To satisfy adaptive, real-time and network-level coordinated TSC, this paper proposes an explicit multiagent coordination (EMC) approach, which models each intersection as an autonomous and cooperative agent. By  \emph{explicit}, we propose a \emph{balance index} to explicitly characterize the coordination efficiency among neighbor intersections. The proposed coordination function can fundamentally align local coordination efficiency with global object of minimizing average travel time of vehicles.  Moreover, EMC only need to access the number of vehicles waiting at the intersections, which is easier to achieve than the fine-grained arrival time of each vehicle required in those schedule-driven online planning methods \cite{VARAIYA2013177,Hua2019}. Indeed, multiagent coordination concept has been considered before for TSC, but in a way that models coordination function as a simple linear weighting of independent agent's reward \cite{Bazzan09,JungesB08,KuyerWBV08,Pol2016} or also models vehicle agents \cite{Iwase0GC22}. In contrast, the proposed EMC only models intersection agents, which can scale to network-level applications.

In summary, our first contribution of this paper is to propose a novel EMC framework to model TSC problem, where a coordination function is designed to regulate signal decisions of neighboring interactions. The second contribution is to propose a decentralized and anytime message-passing algorithm, which can optimize the network-level coordination in a real-time manner. By decentralization, each agent can only exchange messages of coordination function with neighbors. By anytime, the message passing procedure can interrupt at any time when the real time limit is reached and each agent selects the optimal signal decisions by summarizing the current message. Finally, we conduct comprehensive experiments using both synthetic and real road network datasets. Experimental results demonstrate the power of EMC method over traditional transportation heuristics and RL methods with respect to reduce average travel time of vehicles, especially in city-scale scenarios with hundreds of intersections.

The remainder of this paper is organized as follows. In Section 2, we give a brief review of related work on intelligent TSC. In Section 3, we describe the TSC problem and formulate the TSC as a coordination graph-based EMC problem in Section 4. We propose an efficient decentralized and anytime message-passing algorithm in Section 5. In Section 6, we conduct a set of experiments to evaluate our proposed EMC's performance. Finally, we conclude our paper and discuss future work in Section 7.

\section{Related Work}
To address the challenges of dynamic traffic flow, real-time signal control and network level coordination, there have been multiple threads of research dedicated to TSC, including offline multiagent reinforcement learning, and online planning researches such as model-predictive control and conventional transportation methods. In the following, we review and discuss how related researches concern on these two domains. Table \ref{Tab:RelatedWork} gives a brief comparisons of these researches.

\subsection{Offline Reinforcement Learning (Offline RL)}
\textbf{Independent RL.} Because of the highly dynamic traffic flows, reinforcement learning (RL), which can directly learn how to take actions in complex and uncertain environment, has recently been investigated for TSC \cite{6287438}. In RL, each intersection/agent interacts with the traffic environment and learns from rewards to achieve the mapping between the traffic state and the corresponding signal decisions. Thus, careful reward design is essential to minimize network travel time of vehicles \cite{3398777}. Wei et al. \cite{WeiZYL18} weighting multiple factors such as queue length, delay, waiting time, traffic flow and vehicle speed as the immediate reward. Inspired by the max-pressure theory proposed in transportation literature \cite{VARAIYA2013177}, Wei et al. \cite{Hua2019} further propose a "pressure" reward that has the network flow stability guarantee. In terms of scalability, independent RL transfers the single agent's plan to other agents, thereby can scale up to large-scale applications \cite{Bakker2010,ChenWXZYXXL20}. However, some intersections are tightly coupled, without coordination, congestion starts from the upstream intersections will spread to downstream intersections \cite{Bazzan09}.

\textbf{Multiagent RL.} MARL allows agents to communicate and cooperative with neighbor agents, has proposed to optimize network-level coordination. For example, Liu et al. \cite{LiuLC17} propose a distributed RL method by extending the observable surroundings of agents with neighbors' traffic information. Yu et al. \cite{YuLWJHC0020} further propose an active communication mechanism where the historical action and state information is actively communicating to neighbors. The temporal and spatial influences of other non-neighbor agents to the target agent can also be learned by the graph attention networks \cite{WeiXZZZC0ZXL19,WangXNTCX22}. To improve communication efficiency, the impact of signal decisions on its neighbors can be explicitly predicted \cite{JiangQS0Z22}. In order to coordinate policies in an explicit manner, a modular $Q$-learning framework is proposed to select the action that can maximize the sum $Q$ function of each pair of agents \cite{El-TantawyAA13},and Kuyer et al. \cite{KuyerWBV08} propose a factored $Q$-function to learn the coordination policy between a pair of connected agents. Using the centralized training and decentralized execution framework, Chu et al. \cite{Chu2020} propose an advantage actor critic (A2C) mechanism to improve the coordination.

\textbf{Hierarchical RL}. Exploiting the advantage of RL on local connected intersections, hierarchical RL can be applied by decomposing the network into sub-networks, where each sub-network controls its signals to maximize the global performance\cite{Tan2020}. For example, Pol and Oliehoek \cite{Pol2016} employ the max-plus mechanism to coordinate local component actions to get the network-level coordinated policy. The local policy can be better aligned with global centralized $Q$ function with hybrid loss function \cite{ZhangYZ20}. More recently, Ma and Wu \cite{Ma2020} propose a top-down feudal RL where the high-level manager set goals for their workers, while each lower-level worker controls traffic signals to fulfill the managerial goals.

Unfortunately, offline RL approaches have multiple issues that have prevented deployment on real-world TSC: (1) tedious parameter-tuning and an infeasible amount of simulations are required to learn good policies for multiple intersections; (2) since the traffic environment will be non-stationary with multiple agents learning concurrently, deep MARL approaches for TSC are not stable and may trap into sub-optimal solution; (3) learning-based TSC methods can be trained efficiently for regular traffic flows, but difficult to apply in an online manner if traffic flows are changing irregularly (e.g.,unexpected incidents such as accidents and lane closure happen).

\subsection{Online Planning}
Online methods interleave planning and execution by focusing only on states that are reachable from the current state, thus is capable of real time responsiveness for irregular traffic. Max pressure control is an greedy online planning mechanism, where each intersection independently selects to activate the next signal with the maximal pressure (i.e., the difference of queue lengths between upstream and downstream roads) \cite{VARAIYA2013177}. Such \emph{pressure} concept has been used for reward shaping in RL \cite{Hua2019}. While max-pressure is simple to implement, the lack of coordination limits its efficiency in network-level TSC applications.

\textbf{Schedule-Driven Method.} A recent development in online planning that overcomes this coordination problem is schedule-driven TSC \cite{SmithBXR13}. By modeling a cluster of vehicles as indivisible jobs and intersections as machines, the TSC can be formulated as schedule problems where jobs need to be served within the minimum delay \cite{XieSB12}. The downstream intersection can coordinate with the upstream neighbors by receiving and responding outflows from its upstream neighbors. Hu and Smith \cite{HuS19} further extend the basic single direction coordination mechanism to a bi-directional coordination mechanism. Considering the impact of the vehicle queue length on traveling delay, a weight is assigned to each road. For example, Hu and Smith \cite{HuS17} use the link softpressure as the weight and Wu et al. \cite{Wuna2020} use the occupancy ratio as the weight. Given the weight of each road, the signal decision should optimize the weighted cumulative delay. To tackle with the uncertainty associated with the vehicle turn movements at intersections,a sample-based constrained optimization is proposed to minimize expected delay over all vehicles \cite{DhamijaGV020}.

Despite its scalability and effectiveness, key to schedule-driven approaches is an ability to
sense parameters of clusters $(|c|,arr,dep)$, where $c$ is the cluster of vehicles, $arr$ (resp. $dep$) is the expected arrival (resp. departure) time of the first vehicle in the cluster $c$. However, current point sensors can only provide an inexact estimate of arrival times \cite{GoldsteinS18}, where the scheduling efficiency cannot be guaranteed \cite{8443134}. In contrast, our proposed EMC approach only requires the expected number of vehicles arriving at next period, which can be predicted accurately using historical data \cite{LvDKLW15,3511976}.

\textbf{Model-Predictive Control (MPC).}  Equipped with a prediction model describing the traffic dynamics in a finite-time horizon, MPC approaches focus on the online long-term traffic optimization problem \cite{8707103}. By modeling the linear objective, such as minimizing the queue length of roads, an integer program can be applied to generate the real-time strategies \cite{7637098}. In order to optimize the quadratic objective, such as the balance of queue lengths, quadratic-program can be used to achieve the exact solution \cite{ABOUDOLAS2010680}, and approximate dynamic programming is adopted to generate the approximations \cite{Cai09}. On the other hand, due to the nonlinearity of the prediction model, the mixed-integer linear program (MILP) is used to reformulate the MPC problem \cite{LinSXH11}. However, these centralized MPC solutions to traffic regulation require high computational time, and create a single point of failure \cite{Ash18}.  For example, Heung et al. \cite{HeungHF05} propose a centralized coordination mechanisms for dynamically adjusting offsets, which are inherently susceptible to scalability issues with tens of intersections. To scale up, Oliveira et al. \cite{DEOLIVEIRA2010} and Zhou et al. \cite{Zhou17} decompose the centralized MILP  into separable formulations, each can be computed by an independent agent in parallel. For the arterial road with limited number of intersections, a two-way signal coordination method is proposed to achieve bandwidth criterion \cite{7328307}.

However, as the situation current stands, MPC approaches are widely believed not to be able to scale to network-level problems (i.e., take 80 seconds for 34 intersections \cite{Zhou17}). The reason for the limited scalability of MPC is the so-called curse of history: the model prediction behaves like breadth-first search in the combination action space. Compared with MPC researches, in which the centralized control is necessary, we propose a fully decentralized control anytime algorithm that gradually improves the quality of a solution as it runs and return a solution that is not worse than the previous iteration, thereby can scale to network-level applications.

\textbf{Traditional Multiagent Method.} In order to understand how the distributed constraint optimization (DCOP) performs in dynamic changing environments, Jungers and Bazzan \cite{JungesB08} first evaluates DCOP in real-world TSC problems. Kuyer et al. \cite{KuyerWBV08} and Pol and Oliehoek \cite{Pol2016} further combine DCOP and RL to optimize the long-term TSC efficiency. However,the above multiagent methods design a simple linear weighting of heuristic measures (e.g., queue length, delay, waiting time, traffic flow and vehicle speed) as the coordination function between agents, which results in highly sensitive performance \cite{Hua2019}. In contrast, this paper proposes a novel balance index-based coordination function, which can optimize the ultimate objective of minimizing average travel time of vehicles directly. Recently, Iwase et al. \cite{Iwase0GC22} propose to model vehicle agents to control the allocation of slots for a particular vehicle, but only scale-up to small settings with $~$100 vehicles.

\renewcommand{\multirowsetup}{\centering}
\begin{table}[t!]
	\vspace{10pt}
	\centering
	\begin{tabular}{c|c}
		\hline\hline
		Notation  &Description\\  
		\hline\hline
		$\mathcal{N}=\{1,\ldots,i,\ldots,n\}$ & the set of intersections  \\
		$\mathcal{N}_{B}\subseteq \mathcal{N}$  & the set of boundary intersections  \\
		$\mathcal{L}_{all}=\mathcal{L}\cup \mathcal{L}_{entry} \cup \mathcal{L}_{exist}$ & the set of links/roads\\
		$Neg(i)$ & the neighbors of the intersection $i$ \\
		$I(i)$ & the set of links entering $i$\\
		$O(i)$ & the set of links leaving $i$\\
		$l_{ij}$ & the link leaving from $i$ and entering into $j$ \\
		$Do_{l}$ & the set of downstream links of $l$ \\
		$Up_{l}$ & the set of upstream links of $l$ \\
		$(i,h)$ & traffic movement from link $l$ to link $k$\\
		$x(l,h)\in\{0,1\}$ & whether the phase $(l,h)$ is actuated  \\
		$q(l,h)$ & the queue length of phase $(l,h)$ \\
		$f(l,h)$ & the saturation flow of phase $(l,h)$ \\
		$r(l,h)\in[0,1]$ & the proportion of vehicles from $l$ to $h$   \\
		$d(l)$  & the exogenous flow of entry link $l$ \\
		$Q(t)$ & the queue state of road network  \\
		$B(t)$ & the balance index of road network  \\
		\hline\hline
	\end{tabular}
	\vspace{-3pt}
	\caption{Notation Overview} \label{tab:NO}
	\vspace{-15pt}
\end{table}

\section{The TSC Problem Description}
This section presents the TSC problem including the concepts of road network, movement phases, network state update, and the balance index as well as the object of TSC. Table \ref{tab:NO} shows the notations used throughout this paper.

\textbf{Road network.}  Let $G=\langle \mathcal{N}, \mathcal{L}_{all} \rangle$ denote the road network, where $\mathcal{N}=\{1,2,\cdots,n\}$ indicates a set of $n$ signalized intersections, and $\mathcal{L}_{all}=\mathcal{L}\cup \mathcal{L}_{entry} \cup \mathcal{L}_{exist}$ indicates a set of directed links (i.e., road) in the network. There are three types of links: an internal link $l \in \mathcal{L}$ goes from its start intersection $i \in \mathcal{N}$ to its end intersection $j \in \mathcal{N}$; an entry link $l\in \mathcal{L}_{entry}$ has no start intersection; and an exit link $l\in \mathcal{L}_{exist}$ has no end intersection.  An intersection $j$ is a neighbor of the intersection $i$, if there exists a link $l$ whose start (end) and end (start) intersections are $i$ and $j$ respectively. Let $Neg(i)$ denote the neighbors of the intersection $i$. Let $I(i)$ and $O(i)$ denote the set of links entering into (i.e., input) and leaving from (output) the intersection $i$, respectively, and $l_{ij}=I(j)\cap O(i)$ indicates the link leaving from $j$ and entering into $i$. Denoted by the intersection $i$ the boundary intersection if its input links include entry links, i.e., $I(i)\cap \mathcal{L}_{entry} \neq \emptyset$, and by $\mathcal{N}_{B}\subseteq \mathcal{N}$ the set of boundary intersections. Link $l$ is upstream (downstream) of link $h$ if $l$ has an end (start) intersection $i$ and $h$ has an start (end) intersection $i$. Let $Do_{l}$ and $Up_{l}$ denote the the set of downstream and upstream links of the link $l$, respectively.

\textbf{Movement phase.}  A \emph{traffic movement} $(l,h)$ is defined as the traffic traveling across an intersection $i$ from input link $l\in I(i)$ to output link $h\in O(i)$. The set of movement phases at an intersection $i$ is $\{(l,h)\in I(i)\times O(i)\}$. Conflict-free phases at intersection $i$ can be simultaneously activated. The signal control at the intersection $i$ can be represented by $\mathcal{X}_{i}=\{x_{i}(l,h)\in \{0,1\}, l\in I(i), h\in O(i)\}$, where $x_{i}(l,h)=1$ indicates the green light is activated and the traffic movement $(l,h)$ is allowed, and $x_{i}(l,h)=0$ indicates the red light is activated and the traffic movement $(l,h)$ is prohibited. In a typical intersection as shown in Fig. \ref{fig:Framework}, there are 12 ($4\time 3$) traffic movements, which can be controlled by 4 \emph{movement phases}: \emph{WE-Straight} (going straight from West and East), \emph{SN-Straight} (going straight from South and North), \emph{WE-Left} (turning left from West and East), \emph{SN-Left} (turning left from South and North)\footnote{It is worth noting that 1) the "turn right" traffic movement signal can be always activated and 2) there might be different number of phases in the real world intersections and four phases model is not a must}.

\textbf{Network State Update.}  Time is discretized into periods, $t=1,2,\cdots,T$, each includes a fixed duration of $\tau$ seconds (e.g., 10 seconds). Let $q(l,h)$ denote queue length of movement $(l,h)$, i.e., the number of vehicles waiting to leave the link $l$ and enter the link $h$.  Let $f(l,h)$ denote the saturation flow of traffic movement $(l,h)$, i.e., if movement $(l,h)$ is activated at the beginning of the period $t$, there will be at most $f(l,h)$ vehicles traveling from link $l$ to link $h$ at the end of $t$. An inactivated phase serves zero flow of vehicles.  At the beginning of the period $t$, let $Q(t)=\{q(l,h)(t)\}$ denote the queue state of the road network, and $x(l,h)(t)$ denote whether the phase $(l,h)$ is active (i.e., $x(l,h)(t)=1$) or not (i.e., $x(l,h)(t)=0$). According to traffic dynamics \cite{VARAIYA2013177}, the queue state for internal link $l\in \mathcal{L}$ updates as :
\begin{align} \notag \label{Eq:Evo}
&q(l,h)(t+1)=q(l,h)(t)-\overbrace{f(l,h)(t)x(l,h)(t)\wedge q(l,h)(t)}^{outgoing~vehicles}  \\
&+\overbrace{\sum_{e\in Up_{l}}\big[f(e,l)(t)x(e,l)(t)\wedge q(e,l)(t)\big]r(l,h)(t+1)}^{incoming~~vehicles}.
\end{align}
where the function $x\wedge y=\min\{x,y\}$. The second term on the right in Eq.(\ref{Eq:Evo}) indicates that the queue length $x(l,h)(t)$ decreases by up to $f(l,h)(t)$ vehicles during $t$ if $x(l,h)(t)=1$; the third term indicates that up to $f(e,l)(t)$ vehicles will move from upstream queue $(e,l)$ during $t$ if $x(e,l)(t)=1$ and they will join queue $(l,h)$ with a probability of $r(l,h)(t+1)$ at next period $t+1$. This turning proportion parameter $r(l,h)(t+1)$ indicates the proportion of vehicles leaving from $l$ and entering to $h$, which are assumed available by prediction from the route navigation systems \cite{MIRCHANDANI01}.

\begin{figure}[!t]
	\centering
	\includegraphics[width=1.0\linewidth]{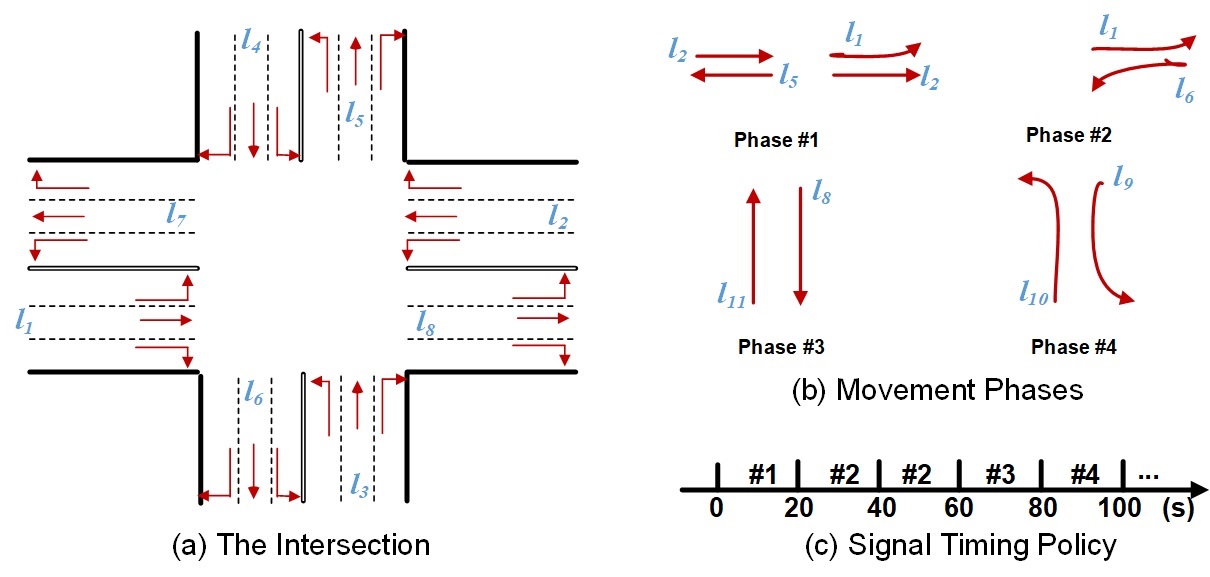}
	\caption{Intersection, movement phases and signal control policy.}
	\vspace{-15pt}
	\label{fig:Framework}
\end{figure}

The queue update equation for an entry link $l\in \mathcal{L}_{entry}$ is a bit different since it has no input links but exogenous arrivals:
\begin{align} \notag \label{Eq:Evo1}
q(l,h)(t+1)&=q(l,h)(t)-\overbrace{f(l,h)(t)x(l,h)(t)\wedge q(l,h)(t)}^{^{outgoing~vehicles}}\\
&+\overbrace{d(l)r(l,h)(t+1)}^{incoming~~vehicles}.
\end{align}
where $d(l)$ denote the exogenous flow of vehicles in entry link $l\in \mathcal{L}_{entry}$, which is also assumed available by estimation, and $r(l,h)(t+1)$ is the turning proposition.

\textbf{The Object.} The ultimate objective of TSC is to minimize the average travel time of vehicles. However, the travel time of vehicle is a long-term object depending on a sequence of signal actions, which cannot be computed until it has completed its route. Existing researches have investigated short-term measure such as a weighted combination of queue length, waiting time and delay \cite{El-TantawyAA13,Pol2016,WeiZYL18} or pressure \cite{VARAIYA2013177,HuS17,Hua2019} as a proxy for the average travel time. However, there is no guarantee that existing short-term measures can preserve the optimal policy. In contrast, this paper proposes a balance index that can align the short-term measure with the long-term travel time objective, defined as follows.
\begin{definition}
	\textbf{The Balance Index.} At period $t$, given the queue state, we utilize the \emph{balance index} $B_{i}(t)$ to characterize the traffic congestion at the intersection $i$, which is defined as the sum of squares of all its movements' queue length, i.e.,
	\begin{align} \label{Eq:pressure}
		B_{i}(t)=\sum_{(l,h):l\in I(i),h\in O(i)}[q(l,h)]^2.
	\end{align}
The network-level balance $B(t)$ can be defined as the sum of these intersections' balance, i.e., $B(t)=\sum_{i\in \mathcal{N}}B_{i}(t)$.
\end{definition}
The intuition behind the balance index is that given a number of vehicles queuing at an intersection, it will produce the less average waiting time in the case of these vehicles are distributed at links in a balance manner than that in the imbalance manner. Moreover, Fig.\ref{Fig:BalanceIndex} empirically shows the alignment of the balance index with the average travel time, i.e., the lower value of the network-level balance index $B(t)$, the less average travel time for vehicles. In the left of Fig.\ref{Fig:BalanceIndex}, we show the balance index of two methods on different datasets, i.e., SYN1, SYN2, HZ and JN (the detail description of these datasets can be seen in experiments), and the right of Fig.\ref{Fig:BalanceIndex} shows the average travel time of these methods. By comparison, the balance index can be used as a proxy measure for system objective (i.e., average travel time).

\begin{figure}[!t]
\begin{tikzpicture}
\begin{axis}[name=plot1,height=4cm,width=4.5cm,
legend style={at={(1.25,-0.5)},anchor=north,legend columns=-0.5},enlargelimits=0.1,
ybar=0pt, bar width=10,xtick align=inside, ymajorgrids, major grid style={dashed, line width=.5pt,draw=gray!50},
symbolic x coords={SYN1,SYN2,HZ,JN},
xtick=data,x label style={at={(0.5,-0.2)}},xlabel={Different Datasets},ylabel={Balance ($\times 10^3$)},y label style={at={(-0.18,0.5)}},ytick={0.5,1,1.5,2,2.5,3.0}]
\addplot[black!60,fill=gray!40]
coordinates{(SYN1,1.126)
	(SYN2,1.828)
	(HZ,0.841)
	(JN,2.154)
};
\addplot+[black!60, pattern=grid]
coordinates{(SYN1,1.236)
	(SYN2,2.058)
	(HZ,1.082)
	(JN,2.706)
};
\legend{Method1,Method2}
\end{axis}

\begin{axis}[name=plot2,at={($(plot1.east)+(3cm,1.2cm)$)},anchor=north,height=4cm,width=4.5cm,enlargelimits=0.1,
ybar=0pt, bar width=10, tick align=inside,ymajorgrids, major grid style={dashed, line width=.5pt,draw=gray!50},
symbolic x coords={SYN1,SYN2,HZ,JN},
xtick=data,x label style={at={(0.5,-0.2)}},xlabel={Different Datasets}, ylabel={Travel Time (s)},y label style={at={(-0.2,0.45)}}]
\addplot+[black!60,fill=gray!40]
coordinates{(SYN1,144)
	(SYN2,180)
	(HZ,345)
	(JN,291)
};
\addplot+[black!60, pattern=grid]
coordinates{(SYN1,155)
	(SYN2,196)
	(HZ,417)
	(JN,342)
};
\end{axis}
\end{tikzpicture}
\caption{The efficiency of balance index.\textbf{LEFT:} the balance index of methods; \textbf{RIGHT:} the average travel time of methods.}
\label{Fig:BalanceIndex}
\end{figure}
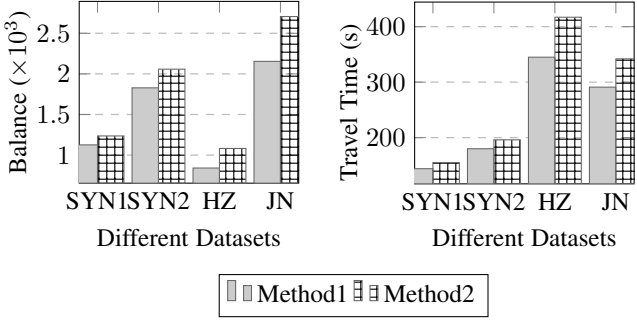%


 Finally, at period $t$, given the current traffic information, including network state $Q(t)$, exogenous demands $d_{l}(t)$, saturation flow $f(l,h)(t)$ and turning proportion $r(l,h)$, the optimal online TSC policy $\vec{x}^{*}(t)=\{x(l,h)^{*}(t)\}$ should be selected with the objective of minimizing the network-level balance at the next period $t+1$, i.e.,
 \begin{align} \label{Eq:objective}
 \vec{x}^{*}(t)=\arg\min_{\vec{x}(t)}B(t+1).
 \end{align}
where the queue state $Q(t+1)$ is updated by the policy $\vec{x}(t)$ on current network state $\{Q(t),f(l,h)(t),d(l)(t),r(l,h)\}$.

\begin{figure*}
	\centering
	\includegraphics[width=0.9\linewidth]{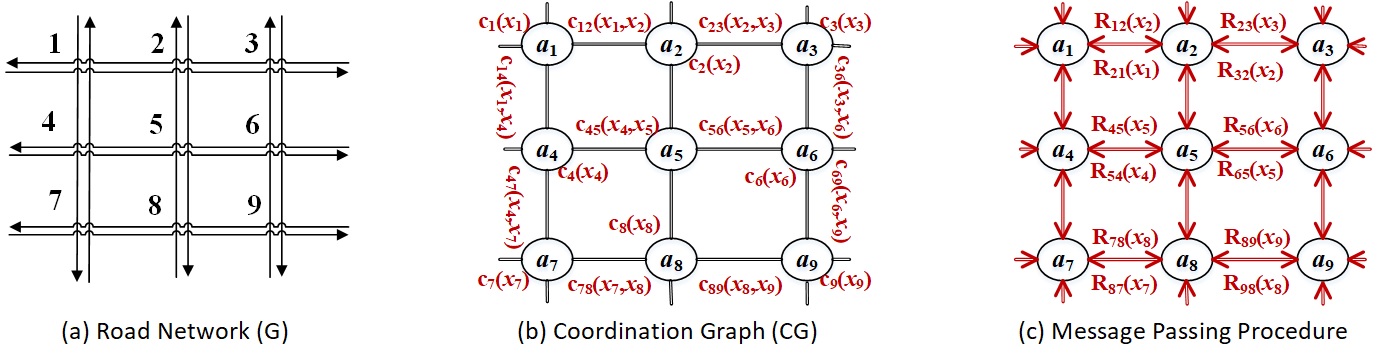}
	\caption{An example of CG formulation for TSC. \textbf{LEFT:} the road network, \textbf{MIDDLE:} the CG formulation. Each intersection $i$ is modeled as an autonomous agent $a_i$, the internal link balance between intersections $i$ and $j$ is modeled by the constraint cost $c_{ij}$, and the entry link balance is modeled by the individual agent cost $c_{i}$,\textbf{RIGHT:}The decentralized message passing procedure.}
	\vspace{-10pt}
	\label{fig:Factorgraph}
\end{figure*}
\section{Multiagent Coordination Model for TSC}
In this section, we design the congruent multiagent coordination graph to model the TSC problem, which will be proven to yield the optimal solution.


\textbf{Coordination Graph (CG).}  A CG is a 5-tuple $(\mathcal{A}, \mathcal{X}, \mathcal{D}, \mathcal{E}, \mathcal{C})$ such that:
\begin{itemize}
	\item $\mathcal{A}=\{a_1,a_2,\cdots,a_n\}$ is the set of agents, each agent has limited computation capacity that can compute and send messages.
	\item $\mathcal{X}=\{x_1,x_2,\cdots,x_n\}$ is the set of variables, each takes the value from the finite discrete domain $D_i \in \mathcal{D}$. Each variable $x_i$ is controlled by the corresponding agent $a_i$. Each agent can have an individual cost function $c_{i}(x_i) \in \mathcal{C}$, which only depends on its own variable $x_i$.
	\item $e_{ij}\in \mathcal{E}$ is the edge connects agents $a_i$ and $a_j$ if they interact with each other. The interaction between connected agents $a_i$ and $a_j$ is formulated by the local constraint cost function $c_{ij}(x_i,x_j) \in \mathcal{C}$, which is defined as a mapping from the assignments of the involved binary variables $(x_{i},x_{j})$ to a positive real, $c_{ij}: D_{i}\times D_{j} \rightarrow \mathbb{R}_{\geq 0}$.
\end{itemize}
Given a CG, the system cost for the global joint action $\vec{x}=\{x_i\}_{a_i\in \mathcal{A}}$ is factored as the sum of individual costs and local edge costs, i.e.,
\begin{align} \label{Eq:Globalcost}
C(\vec{x})=\sum_{a_i\in \mathcal{A}}c_{i}(x_i)+\sum_{e_{ij}\in \mathcal{E}}c_{ij}(x_i,x_j).
\end{align}
The objective of a CG is to find the best joint action, $\vec{x}^{*}$, that minimizes the system cost $C(\vec{x}^{*})$.

\textbf{The CG Formulation for TSC.} The key step toward modeling the TSC problem as a CG is mapping intersection, movement phase, and objective to the 5 tuples in the CG model. By capturing the structure of TSC, a straightforward CG formulation is modeling each agent $a_i$ in CG as an intersection $i$ in TSC\footnote{In practice, the intersection signal controller is equipped with computation unit that can control the traffic light signal \cite{Tora20}.}. Throughout this paper, we use agent and intersection interchangeable. Each agent $a_i$ holds a decision variable $x_i$ determining the movement phase to be activated. The domain of the variable $x_i$ is denoted by $D_{i}=\{\emph{WE-Straight},\emph{WE-Left},\emph{SN-Straight},=\emph{SN-Left}\}$, which corresponds to four movement phases. For two neighboring intersections $i$ and $j\in Neg(i)$, there will be an edge $e_{ij} \in \mathcal{E}$ between agents $a_i$ and $a_j$.

\textbf{Mapping Queue Balance to the Constraint Cost.} Given a particular variable assignment $\vec{x}$, one challenge is to construct constraint cost $\mathcal{C}$ such that when the resulting CG is solved, we obtain a solution that is congruent to the original TSC problem. Considering different kinds of links, e.g., entry link and internal link\footnote{In this paper, we assume that vehicles arrive their destination at the exit link, thus the vehicles on the exit link can be omitted}, we model two kinds of constraint cost,
\begin{itemize}
\item Mapping the balance of the internal link to the edge cost $c_{ij}$: given two neighboring agents $a_i$ and $a_j\in Neg(i)$, we define their local edge cost $c_{ij}$ as the sum balance of bi-direction links $l_{ij}$ and $l_{ji}$ between them, i.e.,
\begin{align}
c_{ij}(x_i,x_j)=\sum_{h\in O(i)}[q(l_{ij},h)]^2+\sum_{h\in O(j)}[q(l_{ji},h)]^2.
\end{align}
\item Mapping the balance of the entry link to the individual agent cost $c_{i}$: for the entry link $l\in \mathcal{L}_{entry}$, we map its link balance to individual cost of its downstream boundary intersection. For each boundary agent $a_i\in \mathcal{A}_{B}$, the individual agent cost $c_i$ is defined as sum queue balance of the input entry link, i.e.,
\begin{align}
c_{i}(x_i)=\sum_{l\in I(i)\cap \mathcal{L}_{entry},h\in O(i)}[q(l,h)]^2.
\end{align}
\end{itemize}
where $\mathcal{A}_{B}$ denote the set of boundary agents (i.e., boundary intersections $\mathcal{N}_{B}$). For any other non-boundary agent $a_i \notin \mathcal{A}_{B}$, the individual cost $c_i$ is set to be zero, i.e., $c_{i}(x_i)=0, \forall a_i \notin \mathcal{A}_{B}, x_i\in D_{i}$.

An example in Fig.\ref{fig:Factorgraph}(a) and Fig.\ref{fig:Factorgraph}(b) show the model. Now we will prove the equivalence of the constructed CG for TSC with respect to optimal solution.
\begin{lemma}
	The object function of the CG is equivalent to the object function of the original TSC.
\end{lemma}
\begin{proof}
	Given any variable assignment $\vec{x}=(x_1,x_2,\cdots,x_n)$, the object function of the CG is
	\begin{align} \notag
	&\sum_{a_i\in \mathcal{A}_{B}}c_{i}(x_i)+\sum_{e_{ij}\in \mathcal{E}}c_{ij}(x_i,x_j) \\  \notag
	&=\sum_{l\in \mathcal{L}_{entry}, h\in Do(l)}[q(l,h)]^2+\sum_{l\in \mathcal{L}, h\in Do(l)}[q(l,h)]^2 \\ \notag
	&=\sum_{l\in \mathcal{L}_{all},h\in Do_{l}}[q(l,h)]^2=\sum_{i\in \mathcal{N}}\sum_{(l,h):l\in I(i),h\in O(i)}[q(l,h)]^2=B.
	\end{align}
\end{proof}

\section{The Online Multiagent Planning Method}
The main idea behind online planning method is that give the current network state $Q(t)$, each agent action selects an action $x_i$ that can optimize the next period network state $Q(t+1)$. Using the CG model, this section proposes a message passing-based online multiagent planning method for TSC, which is fully decentralized and real-time. By \emph{decentralization}, we mean that each agent can only receive message from and send message to neighbor agents. By \emph{real-time}, we mean that each agent can make its signal decision anytime by summarizing the received message when the time limit is reached.

The proposed online multiagent planning method consists of the following two stages:
\begin{itemize}
	\item \emph{network-level coordination} stage, in which agents coordinate their joint actions $\vec{x}^*=\{x_i^*\}_{a_i\in \mathcal{A}}$ to minimize the network balance $B(t+1)$ at the network period $t+1$;
	\item \emph{local individual improvement} stage, in which each agent $a_i$ improves the joint action $x_i^*$ to reduce its own queue balance $B_{i}(t+1)$.
\end{itemize}
\subsection{Network-Level Coordination of Minimizing $B(t+1)$} \label{Sec:GlobalOPT}
In message passing procedure, at each iteration, by summarizing the message received from neighbor agents, each agent sends the message that consists of the cost value with respect to the action variables to neighbors. The network-level coordination can be achieved by iteratively message passing between neighbor agents. Fig.\ref{fig:Factorgraph}(c) gives an illustration of the message passing procedure-based network level coordination.

\subsubsection{Transforming the CG to the directed acyclic graph (DAG)}
In order to avoid the message passing pathology caused by cycles in a CG, we first select an order on all agents, i.e., transforming a cyclic CG to a DAG for message passing.

Existing work \cite{ZivanP12} has constructed the DAG by ordering agents according to the indices of agents, that is an agent $a_i$ is ordered before agent $a_j$ if $i<j$. However, this kind of DAG has an arbitrary diameter, in which the message passing complexity can be further reduced. In this section, we transform a given CG to a DAG with the minimum diameter.

The main idea of constructing DAG is that we first determine a sink agent, and then determine the message passing order of agents according to their distance to the sink agent. A detailed DAG construction is shown in Algorithm \ref{Alg:MiniSO}, which consists of two phases.
\begin{itemize}
	\item Determine the sink agent $a_s$. From steps 2-5, we compute the diameter $dia(a_i)=\max_{a_j \in A}d(a_i,a_j)$ of the DAG if the agent $a_i$ is determined as the sink agent, where $d(a_i,a_j)$ denote the shortest distance between agents $a_i$ and $a_j$. The agent $a_s$ that has the shortest diameter $d(a_s)$ is recorded as the sink agent.
	\item Determine the message passing order. From steps 6-10, once the sink agent $a_s$ is selected, the message passing order can be selected from \emph{faraway} agents (i.e., these agents that are far away from the sink agent $a_s$) to \emph{nearby} agents (i.e., these agents that are close by the sink agent $a_s$) (steps 7-10). For these agents $a_i$ and $a_j$ that have the same distance to the sink agent $a_s$, the message passing order between $a_i$ and $a_j$ is arbitrarily selected.
\end{itemize}

It should be noted that Algorithm \ref{Alg:MiniSO} generates the minimum diameter DAG with a single sink agent.

\begin{algorithm} [!t]
	\SetKwInOut{Input}{Input}
	\SetKwInOut{Output}{Output}
	\Input{The coordination graph $CG$.}
	\Output{The message passing order $o$.}
	Initialize $dia=+\infty$, $o=\emptyset$;\\
	\For{$a_i\in \mathcal{A}$}
	{
		$dia(a_i)=\max_{a_j\in A}d(a_i,a_j)$;  \\
		\If{$dia(a_i)<dia$}
		{
			$dia=dia(a_i)$, $a_s=a_i$;  \tcp{sink agent determination}
		}
	}
	\For{$a_i,a_j\in \mathcal{A}:e_{ij}\in E$}
	{
		\eIf{$d(a_s,a_i)\leq d(a_s,a_j)$}{
			$o_{ij}=a_j\rightarrow a_i$;
		}{
			$o_{ij}=a_i\rightarrow a_j$;
		}
	}
	\caption{Minimum Diameter Directed Acyclic Graph (\textbf{DAG}($o$))}
	\label{Alg:MiniSO}
\end{algorithm}


\subsubsection{Message Passing-based Coordination}
To coordinate the global joint action, we would like to employ the Max-sum\_ADVP (Max-Sum through value propagation on an alternating DAG) algorithm, which has been recently proposed for distributed constraint optimization problem (DCOP) in multiagent domains \cite{1398426,ZivanP12,ChenDW17}.

\textbf{Remarks.} Despite the fact that our CG formulation of TSC is compatible with any complete or incomplete DCOP algorithm, we chose to use Max-Sum\_ADVP as it is one of the fastest and most efficient algorithms in many multiagent domains. It should be noted that the TSC is actually a minimization problem, however, we will continue to refer to it as Max-sum since this name is widely accepted. Compared to traditional Max-sum\_ADVP where the model is given, our main contribution is to taking the DCOP to the real-world TSC problem such as constraint cost and individual cost modeling, and propose a novel and efficient algorithm for TSC.

Given a DAG, let $Neg(i)=\{a_j|e_{ij}\in \mathcal{E}\}$ denote the neighbor agents of $a_i$, $Neg_{foll}(i)\subseteq Neg(i)$ denote the subset of $a_{i}$'s neighbor agents that are ordered after $a_i$, and $Neg_{prev}(i)\subseteq Neg(i)$ denote the subset of $a_{i}$'s neighbor agents that are ordered before $a_i$. During the message passing process, each agent $a_i$ iteratively collects messages from $Neg_{prev}(i)$ and sends messages to $Neg_{foll}(i)$. At each iteration, a message $R_{ij}$ sent from agent $a_i$ to the follow agent $a_j\in Neg_{foll}(i)$ is a scalar-valued function of the action space of receiving agent $a_j$,
\begin{align} \label{Eq:V2Fmessage}
R_{ij}(x_j)=\min_{x_i}\{c_{i}(x_i)+c_{ij}(x_i,x_j)+\sum_{a_k\in Neg_{prev}(i)}R_{ki}(x_i)\}.
\end{align}
This message for each possible value $x_j \in D_j$ can be explained as follows. When agent $a_j$ selects the action $x_j$, agent $a_i$ attempts to coordinate with $a_j$ by selecting his action $x_i$ to minimize the sum of costs incurred by its own cost $c_{i}$, the constraint cost $c_{ij}$ and that it receives from all previous agents $Neg_{prev}(i)$.

Agents exchange messages until convergence such that the received messages do not update. Finally, each agent decides its optimal action by summarizing the messages $R_{ji}$ it receives from its previous agents $a_j\in Neg_{prev}(i)$ in the last iteration. When an agent makes decisions, it accumulates all costs it receives from previous , and selects an action to minimize the sum costs. The action selection procedure for each agent $a_i$ can be formalized by
 \begin{align} \label{Eq:valuedecision}
 x_{i}^{*}=\arg\max_{x_i}\{c_{i}(x_i)+\sum_{j\in Neg_{prev}(i)}R_{ji}(x_i)\}.
 \end{align}

\newcommand{\pushline}{\Indp}
\newcommand{\popline}{\Indm}
\let\oldnl\nl
\newcommand{\nonl}{\renewcommand{\nl}{\let\nl\oldnl}}

\begin{algorithm} [!t]
	\SetKwInOut{Input}{Input}
	\SetKwInOut{Output}{Output}
	\Input{The message passing order $o$.}
	\Output{The joint action $\vec{x}^{*}=\{x_1^*,x_2^*,\cdots,x_n^{*}\}$.}
	\SetKwProg{Fn}{Function}{:}{}
	\While{time budget is remained}
	{
	perform \textsc{Message-Passing}($\langle o,dia(o) \rangle$); \\
	$x_s^{*}\leftarrow$ current optimal decision by Eq.(\ref{Eq:valuedecision});\\
	$o'$ $\leftarrow$ reverse($o$); \\
	perform \textsc{Message-Passing}($\langle o',dia(o') \rangle$); \\
    }	
	Compute the coordinated action $x_i^{*}$ by Eq.(\ref{Eq:valuedecision}); \\	
     \Fn{\textsc{Message-Passing}$(\langle o,dia(o) \rangle)$}
    {	\For{$dia(o)$ \emph{iterations}}
    	{
    	\For{$a_i\in \mathcal{A}$}
    	   {
    	    $Neg_{foll}(i)\leftarrow$ $\{a_j\in Neg(i):a_j$ is before $a_i$ in order $o$ $\}$; \\
         	$Neg_{prev}(i)\leftarrow Neg(i)\setminus Neg_{foll}(i)$; \\
    		Collect messages from $Neg_{prev}(i)$; \\
    		\For{$a_j\in Neg_{foll}(i)$}
    		{
    			send message $R_{ij}$ by Eq.(\ref{Eq:V2Fmessage}) to $a_j$;
    		}
    	  }
        }
    }
	\caption{Network-Level Coordination (\textbf{NL-Coor}$(o,dia(o))$)}
	\label{Alg:Max-Sum-ADE}
\end{algorithm}

We formally propose the network-level coordination (\textbf{NL-Coor}) algorithm in Algorithm \ref{Alg:Max-Sum-ADE}, which includes two phases of \textsc{Message-Passing} procedure i.e., message passing on the forward order $o$ (step 2) and on the reverse order $o' \leftarrow$ reverse($o$) (step 4). The function \textsc{Message-Passing}$(\langle o,dia(o) \rangle)$ presents a sketch of the message passing procedure. At each iteration, each agent $a_i$ first aggregates the message received from "previous" agents $Neg_{prev}(i)$ and sent messages to "follow" agents $Neg_{foll}(i)$ according to Eq.(\ref{Eq:V2Fmessage}). In step 7, the message passing process terminates after $dia(o)$ iterations, where $dia(o)$ denote the diameter of the message passing order $o$.  After \textsc{Message-Passing} converges in both directions or the time limit is reached (step 1), each agent selects its own decision $x_{i}^{*}$ according to Eq.(\ref{Eq:valuedecision}) (step 5).

\begin{lemma} \label{Lemma:Convergence}
	 Given the message passing order $o$, after $dia(o)$ iterations of message passing, the content of the message each agent $a_i$ receives does not change, i.e., the message passing process converges after $dia(o)$ iterations.
\end{lemma}
At the iteration $t$, the content of messages produced by agents is dependent only on the content of messages they received last iteration $t-1$ from previous agents. Starting from the 1st iteration, the boundary agents (that have no previous agents) will always send the constant messages. Starting from the 2nd iteration, the agents that are followed immediately by the boundary agents will always send the constant messages. By induction, after the $dia(o)$ iterations, the sink agent will receive the constant messages.

\begin{lemma}
	The message passing procedure \textsc{Message-Passing} converges the fastest on the DAG generated by Algorithm \ref{Alg:MiniSO}.
\end{lemma}
Derived from the Lemma \ref{Lemma:Convergence} that \textsc{Message-Passing} procedure converges within $dia(o)$ iterations. Since Algorithm \ref{Alg:MiniSO} generates a DAG with the minimum diameter, we can conclude that \textsc{Message-Passing} procedure will converge with the minimum iterations.

\subsubsection{Stability Analysis}
In this section, we show the proposed NL-Coor mechanism stabilizes the network. Stability is an important metric to evaluate the TSC policy in transportation domain, which is defined as follows.
\begin{definition}
	(Network stability). The network state process $Q(t)=\{q(l,h)(t)\}$ is stable in the mean (and $\vec{x}$ is a stabilizing TSC policy) if for some $K < \infty$
	\begin{align}
	\frac{1}{T}\sum_{t=1}^{T}\sum_{l,h}\mathbb{E}[q(l,h)(t)]<K, \forall T.
	\end{align}
	where $\mathbb{E}$ denotes the expectation.
\end{definition}
Varaiya \cite{VARAIYA2013177} has proposed the \emph{MaxPressure} TSC policy can guarantee the network stability.
\begin{definition}
	MaxPressure \cite{VARAIYA2013177}. Given the network state $Q(t)$ at current period $t$, each agent/intersection selects the action with the maximum pressure: $\vec{x}^{*}(t)=\arg\max_{\vec{x}(t)}\sum_{(l,h):x(l,h)=1}f(l,h)[q(l,h)-\sum_{p\in Do_{h}}r(h,p)q(h,p)]$, where $q(l,h)$ is the upstream queue length $\sum_{p\in Do_{h}}r(h,p)q(h,p)$ is the (average) downstream queue length, the pressure $q(l,h)-\sum_{p\in Do_{h}}r(h,p)q(h,p)$ of a movement phase is simply the difference between upstream and downstream queue lengths.
\end{definition}

\begin{figure}[!t]
	\centering
	\begin{tikzpicture}[xscale=0.8,yscale=0.8]
		\begin{axis}[
			xlabel=Time $t$,
			ylabel=The Balance index $B(t)$,
			every axis plot/.append style={ultra thick},
			ymajorgrids, major grid style={dashed, line width=.5pt,draw=gray!50}
			]
			\addplot+ table[black, mark=none,x=time, y=dcop_evaluation] {hangshou_4x4_evaluation.dat};
			\addplot+ table[mark=none, x=time, y=mp_evaluation,dotted] {hangshou_4x4_evaluation.dat};
			\legend{NL-Coor (ours), MaxPressure}
		\end{axis}
	\end{tikzpicture}
	\caption{The balance index between our NL-Coor and MaxPressure approaches.} \label{fig:Stability}
\end{figure}
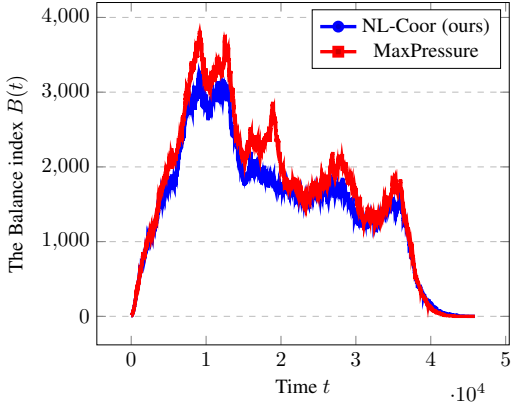

\begin{figure*}[!t]
	\begin{subfigure}[t]{0.33\textwidth}
		\centering
		\includegraphics[width=1\linewidth]{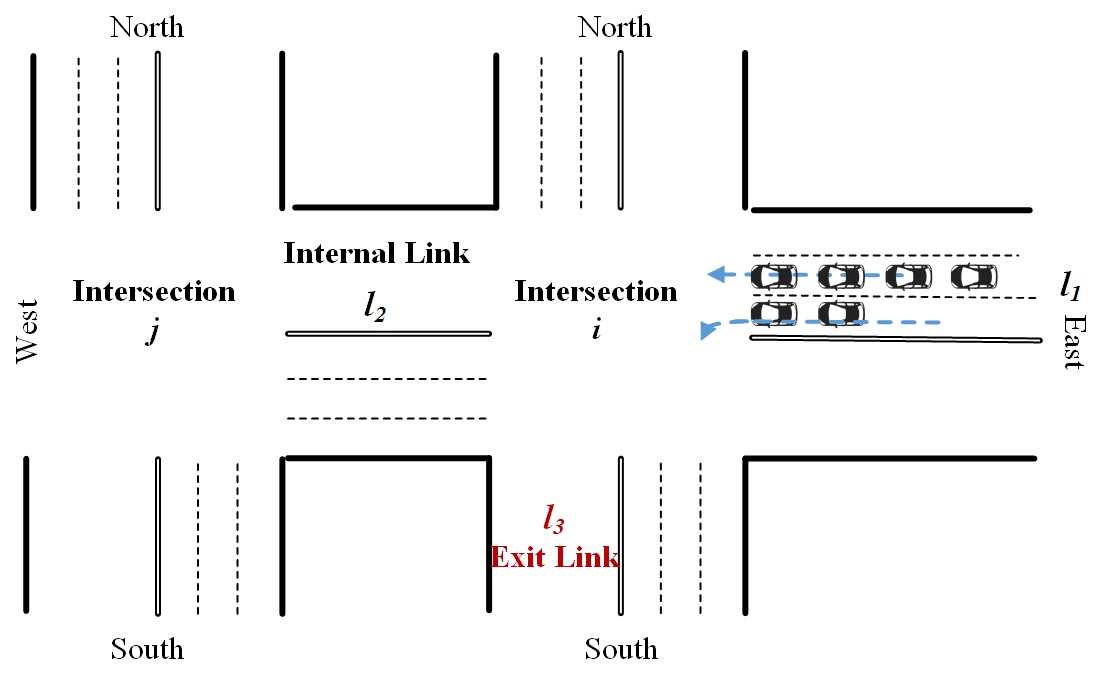}
		\caption{A road network with two intersections.}
	\end{subfigure}
	\begin{subfigure}[t]{0.33\textwidth}
		\centering
		\includegraphics[width=1\linewidth]{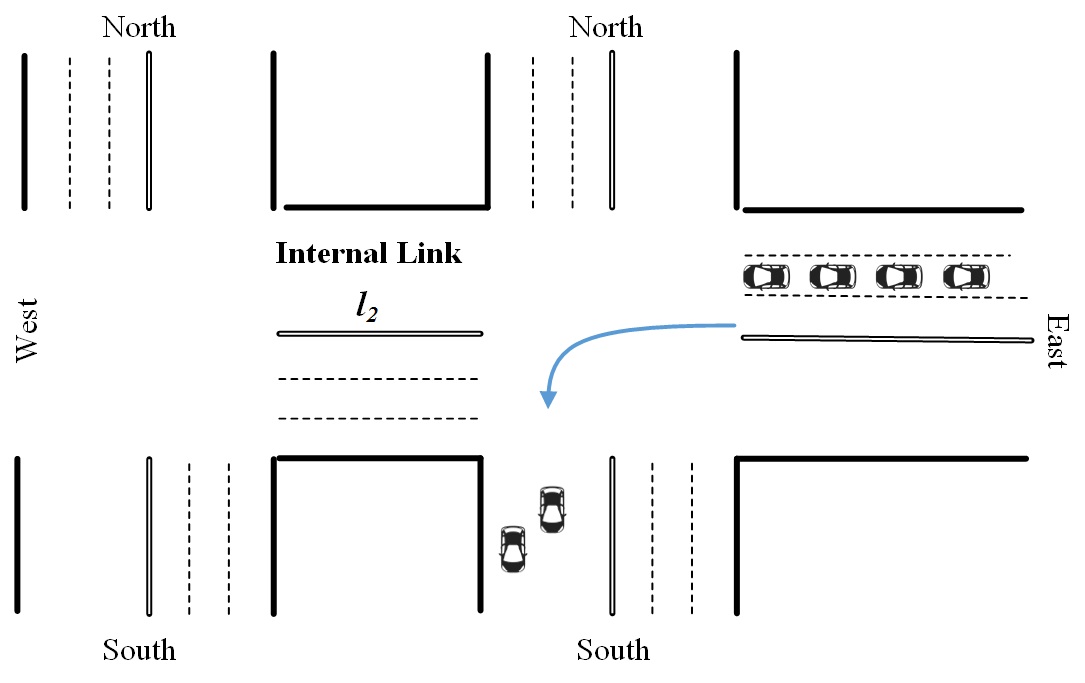}
		\caption{Network-Level Coordinated Policy}
	\end{subfigure}
	\begin{subfigure}[t]{0.33\textwidth}
		\centering
		\includegraphics[width=1\linewidth]{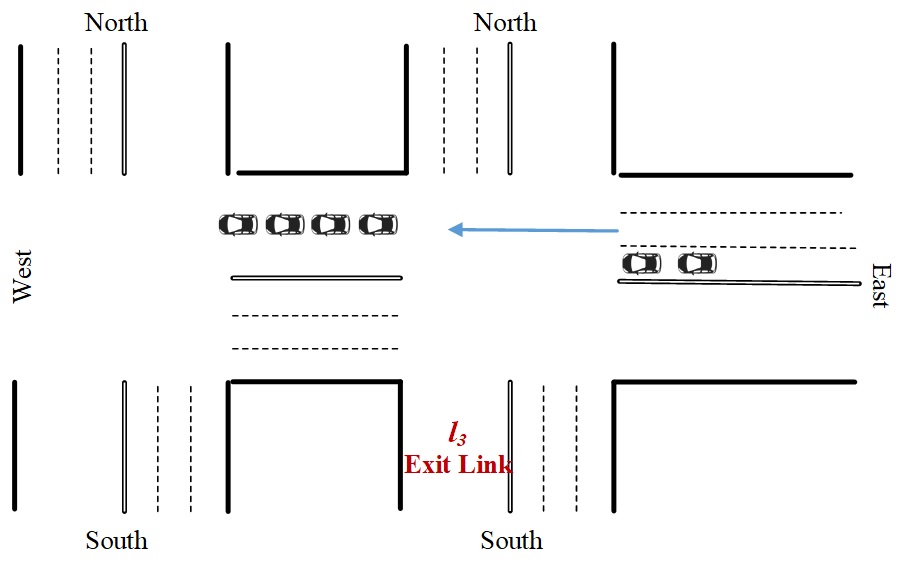}
		\caption{Greedy Individual Policy}
	\end{subfigure}
	\caption{The drawback of NL-Coor and the advantage of individual action on reducing travel time of vehicles.}
	\label{fig:GreedyMotivation}
\end{figure*}

The main idea of MaxPressure policy \cite{VARAIYA2013177} of stability proof is to guarantee the sufficient condition 	
\begin{align} \label{Eq:stability}
\mathbb{E}[B(t+1)]-\mathbb{E}[B(t)]\leq K-\epsilon \mathbb{E}|Q(t)|.
\end{align}
where $|Q(t)|=\sum_{(l,h)}q(l.h)$ denote the sum queue length of road network at the period $t$, $\epsilon >0$ and $K< \infty$.

In the following theorem, we prove that our NL-Coor can stabilize the road network by satisfying such as sufficient condition Ineq.(\ref{Eq:stability}).
\begin{theorem}
	The policy $\vec{x}^{NL-Coor}$ selected by our network-level coordination algorithm (i.e., Algorithm \ref{Alg:Max-Sum-ADE}) can stabilize the network queues.
\end{theorem}
\begin{proof}
	Given the network state $Q(t)$ at the period $t$, let $\vec{x}^{MP}$ denote the signal control policy selected by MaxPressure and $\vec{x}^{Opt}$ denote the optimal policy that minimizes the sum of squares of all the queue lengths at the next period $t+1$, i.e., $\vec{x}^{Opt}=\arg\min_{\vec{x}}\mathbb{E}B(t+1)$. Let $\mathbb{E}[B^{Opt}(t)]$, $\mathbb{E}[B^{NL-Coor}(t+1)]$ and $\mathbb{E}[B^{MP}(t+1)]$ denote the sum of squares of all the queue lengths returned by the optimum, our NL-Coor and the MaxPressure algorithms. On the one hand, we have that $\mathbb{E}[B^{Opt}(t+1)] \leq  \mathbb{E}[B^{MP}(t+1)]$. On the other hand, in acyclic graphs, NL-Coor algorithm can achieve the optimum, and in generalized cyclic graphs, NL-Coor algorithm is an adequately approximation of the optimum \cite{ZivanP12,ChenDW17}, i.e., $\mathbb{E}[B^{NL-Coor}(t+1)] \approx \mathbb{E}[B^{Opt}(t+1)]\leq \mathbb{E}[B^{MP}(t+1)]$. In Figure \ref{fig:Stability}, by empirical evaluation on real datasets, we further validate that  $\mathbb{E}[B^{NL-Coor}(t+1)]\leq \mathbb{E}[B^{MP}(t+1)]$. Finally, we can derive that
\begin{align} \notag
	&\mathbb{E}[B^{NL-Coor}(t+1)]-\mathbb{E}[B(t)] \\
	&\leq \mathbb{E}[B^{MP}(t+1)]-\mathbb{E}[B(t)] \leq K-\epsilon \mathbb{E}|Q(t)|
\end{align}
which can guarantee our NL-Coor algorithm stabilizes the network queues.
\end{proof}

\begin{algorithm}[!t]
	\SetKwInOut{Input}{Input}
	\SetKwInOut{Output}{Output}
	\Input{The coordination graph $G$, the network-level coordination policy $\vec{x}^{NL-Coor}$.}
	\Output{The individual action $\vec{x}=\{x_1,x_2,\cdots,x_n\}$.}
	Initialize the individual action as $\vec{x}(0)=\vec{x}^{NL-Coor}$; \\
	\While{time budget is remained}
	{
		\For{each message passing iteration $\tau$}
		{
			\For{$a_i\in \mathcal{A}$}
			{
				send current action message $x_{i}(\tau)$ to neighbors $Neg(i)$; \\
				$\vec{x}_{Neg(i)}(\tau)$ $\leftarrow$ action messages collected from $Neg(i)$; \\
				Select the optimal action $x_i(\tau+1)=\arg\min_{x_i\in D_{i},\vec{x}_{Neg(i)}(\tau)}B_{i}$.
			}
		}	
	}
	\caption{Local Individual Action Improvement(\textbf{Loc-IAI}($\vec{x}^{NL-Coor}$))}
	\label{Alg:DSA-Greedy}
\end{algorithm}
\subsection{Local Greedy Action of Improving $B_{i}$} \label{Sec:IndividualOPT}
As shown in Section 5.1, the main object of NL-Coor is to minimize the network balance $B$. From the network-level perspective, these vehicles traveling between internal links do not reduce network balance, the proposed NL-Coor might struggle to drive vehicles to the exit links for network balance minimization. However, such "clean" policy might increase the travel time of other vehicles queuing at the internal links. We take an example to illustrate this scenario.

\textbf{Example 1.} \emph{Fig.\ref{fig:GreedyMotivation}(a) shows a simple road network including two intersections $i$ and $j$, an entry link $l_1$, an internal link $l_2$ and an exit link $l_3$. Assume that there are four vehicles queuing at the movement phase $(l_1,l_2)$ and two vehicles queuing at the movement phase $(l_1,l_3)$. From the network-level perspective of minimizing the network balance, NL-Coor will activate the WE-Left movement phase of the intersection $i$ such that these vehicles on $(l_1,l_3)$ can exit the network, and the vehicles on $(l_1,l_2)$ are still waiting for next periods, where the network queue balance is $4^2=16$ (i.e., Fig.\ref{fig:GreedyMotivation}(b)). Otherwise, if the WE-Straight movement phase is activated, these vehicles queuing at the movement phase $(l_1,l_2)$ will move to internal link $l_2$ and the vehicles on $(l_1,l_3)$ are stilling waiting at the the entry link $l_1$ (i.e., Fig.\ref{fig:GreedyMotivation}(c)), where the network queue balance is $4^2+2^2=20$. However, we can find that activate the WE-Straight movement phase will reduce the average travel time of vehicles since more vehicles (i.e., these vehicles queuing at the movement phase $(l_1,l_2)$) will get closer to their destinations.}

To avoid this "clean" policy caused by NL-Coor mechanism, we propose a local action improvement procedure from the perspective of each individual agent. This local action improvement mechanism is also based on message passing, but is associated only with the action information. After collecting the neighbors' actions, each agent $a_i$ greedily selects the individual action with the aim of minimizing its own balance $B_{i}=\sum_{(l,h):l\in I(i),h\in O(i)}[q(l,h)]^2$.

Algorithm \ref{Alg:DSA-Greedy} formally describes the local individual action improvement (Loc-IAI) procedure. In steps 5-6, each agents $a_i$ sends the action information to and collect the action information from neighbors $Neg(i)$ (The neighbor action information can be achieved either in a synchronization or asynchronous manner \cite{ZhangWXW05}). Given the current action vector of neighbors $\vec{x}_{Neg(i)}(\tau)$ of the current message passing iteration $\tau$, each agent $a_i$ selects the action $x_{i}(\tau+1)$ that can minimize its balance $B_{i}$ in a best-response manner (step 7). The message passing-based local individual agent improvement procedure can stop at any time if the time limit is reached.

\begin{algorithm}[!t]
	\SetKwInOut{Input}{Input}
	\SetKwInOut{Output}{Output}
	\Input{The coordination graph $CG$,real-time budget $B^{T}$, the parameter $\epsilon$.}
	\Output{The action vector $\vec{x}=\{x_1,x_2,\cdots,x_n\}$.}
	$\vec{x}^{NL-Coor} \leftarrow$ NL-Coor($o$,$dia(o)$,$\epsilon B^{T}$); \\
	$\vec{x}\leftarrow $ Loc-IA($\vec{x}^{NL-Coor}$,1-$\epsilon B^{T}$);
	\caption{Explicitly Multiagent Coordination Algorithm (EMC)}
	\label{Alg:EMC}
\end{algorithm}

\subsection{Explicitly Multiagent Coordination Algorithm}
So far, we have discussed the two phases of network-level coordination procedure (i.e, NL-Coor) and local action improvement procedure (i.e., Loc-IAI). In this section we integrate these two phases to propose the explicitly multiagent coordination algorithm (EMC) with time budget limits. Given the real-time budget $B^{T}$ (e.g., 3s), in the first coordination phase, NL-Coor executes with respect to the time budget $\epsilon B^{T}$, where $\epsilon \in [0,1] $ is the budget allocation parameter, and the remained $1-\epsilon B^{T}$ time budget is allocated for the second individual action improvement procedure Loc-IAI. The EMC is formally described in Algorithm \ref{Alg:EMC}.

\section{Experimental Evaluation}
\subsection{Experiment Setup}

We evaluate the proposed TSC methods on the Cityflow simulation platform \cite{Zhang21} within real-world and synthetic traffic networks. We estimate queue update of next period by simulating the real Cityflow and use it as the model input of our EMC method. The vehicle's traveling speed is set as a usual of 10m/s. In each traffic network, each vehicle is described as $(o,t,d)$, where $o$ is the origin location (i.e., link), $t$ is time, and $d$ is destination location. Locations $o$ and $d$ are both links of the road network. Unless otherwise specified, each intersection is set to be a bi-direction four-way intersection.

\textbf{Synthetic Road Network.} We generate three synthetic datasets, and their detail statistics are shown in Table \ref{Tab:syn}. These large datasets $15\times 15$ (i.e., 225 intersections) and $20\times 20$ (i.e., 400 intersections) are used to validate the effectiveness (i.e., real-time) and efficiency (i.e., average travel time of vehicles) of our EMC method. The vehicles enter the road network through the boundary intersections and follow a uniform arrival rate. The traffic flow is simulated for 1 hour (i.e., the duration is 3600s). The length of the link H-$k$-V-$l$ indicates that the length of the link in the horizon (H) direction is $k$ meters and the length of link in the vertical (V) direction is $l$ meters.

\textbf{Real-World Road Network.} Four representative real-world traffic datasets collected from four cities \cite{WeiXZZZC0ZXL19} are used to evaluate the performance of methods. Detailed descriptions on these datasets are as follows, with data statistics listed in Table \ref{Tab:real}.

\begin{table}[!t]
	\centering
	\begin{tabular}{c|c|c|c|c}
		\hline\hline
		\multirow{2}* {Dataset}  & Number of   & Arrival Rate & Duration & Length\\
		&Intersections  &(vehicles/s) & (s)  & of the Lane (m)\\
		\hline\hline
		syn\_4 $\times$ 4  & 16 & 1.76 & 3600 &H-300-V-300\\
		syn\_15$\times$ 15  & 225 & 0.80 & 3600 &H-300-V-300\\
		syn\_20$\times$ 20 & 400 & 0.77 & 3600 &H-300-V-300\\
		\hline\hline
	\end{tabular}
	\caption{Data statistics of synthetic traffic dataset.}\label{Tab:syn}
	\vspace{5pt}
\end{table}
\begin{table} [!t]
	\centering
    \begin{tabular}{c|c|c|c|c}
    	\hline\hline
    	\multirow{2}* {Dataset}  & Number of  & Arrival Rate & Duration & Length\\
    	&Intersections  &(vehicles/s) & (s)  & of the Lane (m)\\
    	\hline\hline
    	NewYork  & 1$\times$16 & 0.80 & 3600 & H-350-V-100\\
    	Manhattan & 3$\times$16 & 0.77 & 3600 & H-350-V-100\\
    	Jinan  & 4$\times$3 & 0.84 & 3600 & H-800-V-400\\
    	Hangzhou  & 4$\times$4 & 1.76 & 3600 & H-800-V-600 \\
	\hline\hline
	\end{tabular}
	\caption{Data statistics of real-world traffic dataset.}\label{Tab:real}
\end{table} 

\begin{table*}[t!]
	\centering
	\begin{tabular}{c|c||ccccccc}
		\hline\hline
		\multicolumn{2}{c||}{\multirow{2}*{\backslashbox{Method}{Dataset}}}  & \multirow{2}*{ny\_$1\times 16$} & \multirow{2}*{manhattan\_$3\times 16$} & \multirow{2}* {jinan\_$3\times4$} & \multirow{2}*{hangzhou\_$4\times 4$} &
		 \multirow{2}* {syn\_$4\times 4$}  & \multirow{2}* {syn\_$15\times 15$} & \multirow{2}* {syn\_$20\times 20$}   \\
		\multicolumn{2}{c||}{} & & & & & & &\\
		\hline \hline
		\parbox[t]{5mm}{\multirow{4}{*}{\rotatebox[origin=c]{270}{Heuristics}}}& \multirow{2}*{FixedTime}   & \multirow{2}*{$<$0.1}  & \multirow{2}*{$<$0.1} & \multirow{2}*{$<$0.1} & \multirow{2}*{$<$0.1} & \multirow{2}*{$<$0.1} & \multirow{2}*{$<$0.1} & \multirow{2}*{$<$0.1}  \\
		&\multicolumn{1}{c||}{} & & & & & & &\\
		\cline{2-9}
		&\multirow{2}*{MaxPressure}  & \multirow{2}*{$<$0.1} & \multirow{2}*{$<$0.1} & \multirow{2}*{$<$0.1} & \multirow{2}*{$<$0.1} & \multirow{2}*{$<$0.1} & \multirow{2}*{$<$0.1} & \multirow{2}*{$<$0.1} \\
		&\multicolumn{1}{c||}{} & & & & & & & \\
		\hline
		\parbox[t]{4mm}{\multirow{8}{*}{\rotatebox[origin=c]{270}{RL Methods}}} &\multirow{2}*{LIT}   & \multirow{2}*{$<$0.1} & \multirow{2}*{$<$0.1} & \multirow{2}*{$<$0.1} & \multirow{2}*{$<$0.1} & \multirow{2}*{$<$0.1} & \multirow{2}*{$<$0.1} & \multirow{2}*{$<$0.1} \\ &\multicolumn{1}{c||}{} & & & & & & & \\
		\cline{2-9}
		&\multirow{2}*{PressLight}   & \multirow{2}*{$<$0.1} & \multirow{2}*{$<$0.1} & \multirow{2}*{$<$0.1} & \multirow{2}*{$<$0.1} & \multirow{2}*{$<$0.1} & \multirow{2}*{$<$0.1} & \multirow{2}*{$<$0.1} \\ &\multicolumn{1}{c||}{} & & & & & & &\\
		\cline{2-9}
		&\multirow{2}*{MA2C}   & \multirow{2}*{$<$0.1} & \multirow{2}*{$<$0.1} & \multirow{2}*{$<$0.1} & \multirow{2}*{$<$0.1} & \multirow{2}*{$<$0.1} & \multirow{2}*{$<$0.1} & \multirow{2}*{$<$0.1}  \\ &\multicolumn{1}{c||}{} & & & & & & &\\
		\cline{2-9}
		&\multirow{2}*{FMA2C}   & \multirow{2}*{$<$0.1} & \multirow{2}*{$<$0.1} & \multirow{2}*{$<$0.1} & \multirow{2}*{$<$0.1} & \multirow{2}*{$<$0.1} & \multirow{2}*{$<$0.1} & \multirow{2}*{$<$0.1}  \\ &\multicolumn{1}{c||}{} & & & & & & &\\
		\hline
		\multicolumn{2}{c||}{\multirow{2}*{EMC (ours)}}  & \multirow{2}*{0.174} & \multirow{2}*{0.268} & \multirow{2}*{0.155}& \multirow{2}*{0.165} & \multirow{2}*{0.169} &  \multirow{2}*{1.2765} & \multirow{2}*{2.013}   \\
		\multicolumn{2}{c||}{} & & & & & & &\\
		
		\hline \hline
	\end{tabular}
	\caption{Performance comparison between methods on computation time (seconds). Each cell is statistically significant at 95\% confidence level.}\label{Tab:avgRes}
\end{table*} 

\begin{table*}[t!]
	\centering
	\begin{tabular}{c|c||ccccccc}
		\hline\hline
		\multicolumn{2}{c||}{\multirow{4}{*}{\backslashbox{Method}{Dataset}}}  & \multicolumn{2}{c|}{\multirow{2}{*}{Short Link}} & \multicolumn{2}{c|}{\multirow{2}{*}{Long Link}} & \multicolumn{3}{c}{\multirow{2}{*}{Medium Link}}
		\\
		\multicolumn{2}{c||}{} &\multicolumn{2}{c|}{} & \multicolumn{2}{c|}{} & \multicolumn{3}{c}{} \\
		\cline{3-9}
		\multicolumn{2}{c||}{}  &\multirow{2}{*}{ny\_$1\times 16$ } & \multicolumn{1}{c|}{\multirow{2}{*}{manhattan\_$3\times 16$}} & \multirow{2}{*}{jinan\_$3\times 4$} & \multicolumn{1}{c|}{\multirow{2}{*}{hangzhou\_$4\times 4$}}  & \multirow{2}{*}{syn\_$4 \times 4$}  & \multirow{2}{*}{syn\_$15 \times 15$} & \multirow{2}{*} {syn\_$20\times 20$}  \\
		\multicolumn{2}{c||}{} & &\multicolumn{1}{c|}{} & &\multicolumn{1}{c|}{} & & & \\
		\hline\hline
		\parbox[t]{5mm}{\multirow{4}{*}{\rotatebox[origin=c]{270}{Heuristics}}}  & \multirow{2}*{FixedTime} & \multirow{2}*{199.5($\pm$0.1)} & \multirow{2}*{292.3($\pm$0.7)} & \multirow{2}*{331.2($\pm$1.1)} & \multirow{2}*{377.5($\pm$1.4)}  & \multirow{2}*{232.2($\pm$1.2)} & \multirow{2}*{207.1($\pm$1.0)} & \multirow{2}*{331.9($\pm$1.5)}  \\
		&\multicolumn{1}{c||}{} & & & & & & &\\
		\cline{2-9}
		&\multirow{2}*{MaxPressure}   & \multirow{2}*{\textbf{102.8($\pm$0.5)}} & \multirow{2}*{\textbf{180.4($\pm$0.9)}} & \multirow{2}*{342.4($\pm$1.3)} & \multirow{2}*{416.7($\pm$1.5)}  & \multirow{2}*{196.6($\pm$0.7)} & \multirow{2}*{184.5($\pm$0.6)} & \multirow{2}*{269.2($\pm$1.0)} \\
		&\multicolumn{1}{c||}{} & & & & & & & \\
		\hline
		\parbox[t]{4mm}{\multirow{8}{*}{\rotatebox[origin=c]{270}{RL Methods}}} &\multirow{2}*{LIT}   & \multirow{2}*{103.1($\pm$0.1)} & \multirow{2}*{295.8($\pm$0.1)} & \multirow{2}*{316.4($\pm$0.1)} & \multirow{2}*{362.2($\pm$0.2)} & \multirow{2}*{184.6($\pm$0.1)}  & \multirow{2}*{203.59($\pm$0.2)} & \multirow{2}*{302.0($\pm$0.2)} \\
		&\multicolumn{1}{c||}{} & & & & & & & \\
		\cline{2-9}
		& \multirow{2}*{PressLight}   & \multirow{2}*{106.2($\pm$0.1)} & \multirow{2}*{182.4($\pm$0.8)} & \multirow{2}*{301.3($\pm$8.9)}
		& \multirow{2}*{365.0($\pm$11.4)}  & \multirow{2}*{184.8($\pm$3.5)} & \multirow{2}*{206.6($\pm$3.98)} & \multirow{2}*{345.5($\pm$0.8)}  \\
		&\multicolumn{1}{c||}{} & & & & & & &\\
		\cline{2-9}
		& \multirow{2}*{MA2C}  & \multirow{2}*{115.4($\pm$1.3)} & \multirow{2}*{244.1($\pm$0.5)} & \multirow{2}*{306.4($\pm$0.3)} & \multirow{2}*{364.1($\pm$1.8)}  & \multirow{2}*{205.02($\pm$0.5)} & \multirow{2}*{244.8($\pm$0.7)} & \multirow{2}*{330.8($\pm$2.7)} 	 \\
		&\multicolumn{1}{c||}{} & & & & & & &\\
		\cline{2-9}
		&\multirow{2}*{FMA2C} & \multirow{2}*{120.5($\pm$2.5)} & \multirow{2}*{241.2($\pm$1.0)} &	\multirow{2}*{304.6($\pm$0.9)} & \multirow{2}*{361.1($\pm$1.3)}  & \multirow{2}*{199.6($\pm$0.8)} & \multirow{2}*{241.2($\pm$0.5)} & \multirow{2}*{315.6($\pm$3.0)} 	\\
		&\multicolumn{1}{c||}{} & & & & & & &\\
		\hline
		\multicolumn{2}{c||}{\multirow{2}*{EMC (ours)}}  & \multirow{2}*{104.3($\pm$0.1)} & \multirow{2}*{182.5($\pm$1.4)} &  \multirow{2}{*}{ \textbf{290.9($\pm$0.1)}} & \multirow{2}*{\textbf{355.1($\pm$0.8)}}  & \multirow{2}*{\textbf{180.3($\pm$0.2)}} & \multirow{2}*{\textbf{178.6($\pm$0.1)}} & \multirow{2}*{\textbf{251.6($\pm$0.1)}}  \\
		\multicolumn{2}{c||}{} & & & & & & &\\
		\hline\hline
	\end{tabular}
	\caption{Performance comparison between methods on average travel time (the lower the better). The short, medium, and long link indicates the length of the link is short, medium and long.}\label{Tab:avgExpTime}
\end{table*} 

\begin{itemize}
\item \textbf{ny\_$1\times 16$}: Flow data of 16 intersections from 8-th Avenue in New york City with uni-directional traffic on both the arterial and the side streets.
\item \textbf{manhattan\_$3\times 16$}: Flow data generated from open-source taxi trip data of 48 intersections in the road network of Manhattan, New York City.
\item \textbf{jinan\_$3\times 4$}: Flow data collected by roadside cameras from 12 intersections in Dongfeng Sub-district, Jinan, China.
\item \textbf{hangzhou\_$4\times 4$}: Flow data of 16 intersections in Gudang Sub-district generated from roadside surveillance cameras.
\end{itemize}

\textbf{Comparison Methods.} We compare our EMC method with the following two categories of methods: heuristic methods and RL methods. Heuristic methods mainly design TSC by expert domain knowledge, while RL methods mainly learn TSC by interacting with traffic environment.

\noindent \emph{\textbf{Heuristic Methods:}}
\begin{itemize}
	\item \textbf{FixedTime }\cite{Lowrie1990SCATSSC,Robertson91}: The traffic light cycles through a fixed sequence of phases. The signal sequence is pre-determined by expert rule. Each phase has a fixed duration of 10 seconds. In this paper, there are four phases, i.e., WE-Straight, WE-Left, SN-Straight, SN-Left. Due to its simpleness, FixedTime method has been widely used in practice.
	\item \textbf{Maxpressure} \cite{VARAIYA2013177}: Similar to our EMC method, the Maxpressure method is also an online planning method, which observes real traffic and greedily active the phase with the maximum pressure. A brief introduction of Maxpressure in shown in Section 5.1.3.
\end{itemize}
\emph{\textbf{RL Methods:}}
\begin{itemize}
\item \textbf{PressLight} \cite{Hua2019}: PressLight method models each intersection as an independent agent and learns the policy of choosing next phase by vanilla DQN. The state is set as the queue length on links and reward is set as the negative of the pressure.
\item \textbf{LIT} \cite{abs-1905-04716}: LIT method is also an individual deep RL method which uses phase and vehicles numbers as the state, and the sum of queue length as the reward. Compared with PressLight, it does not consider the traffic condition from downstream.
\item \textbf{MA2C} \cite{Chu2020}: Multi-agent A2C is a mutil-agent RL method, which uses information of neighborhood policies to enhance the observation of each local agent.
\item \textbf{FMA2C} \cite{Ma2020}: FMA2C extends MA2C with the feudal hierarchy, where each intersection is controlled by a worker agent and a each region is controlled by a manager agent. In the two-level manager-worker framework, managers coordinate their high-level behaviors and set goals for their workers in the region, while each lower-level worker controls traffic signals
to fulfill the managerial goals.
\end{itemize}

Note that all methods are carefully tuned and their best results are reported. For a fair comparison, all the RL methods are learned without pre-trained process and the action interval is set as 10 seconds. All RL methods are trained using 500 episodes. 

\textbf{Performance Metric.} We use the computation time (in seconds) of making signal decisions \footnote{The proposed EMC method is parallelizable and hence the running time could be reduced substantially using multiple cores.} and the average travel time of vehicles to evaluate the performance of methods. The average travel time is the widely used measure in TSC literature, which is calculated as the average travel time of all vehicles spent in the road network.

All computations are performed on a 64-bit workstation with 64 GB RAM and a 16-core 3.5 GHz processor. All records are averaged over 40 instances, and each record is statistically significant at 95\% confidence level unless otherwise specified.

\begin{figure*}[!ht]
	\begin{subfigure}[t]{0.34\textwidth}
		\centering
		\includegraphics[width=1\linewidth]{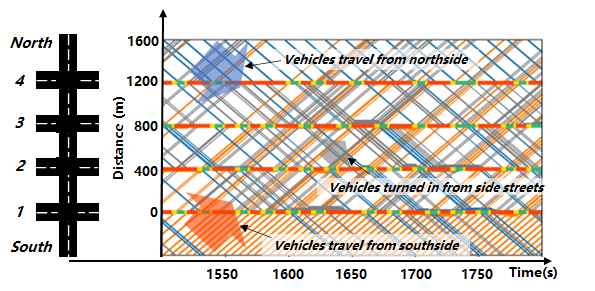}
		\caption{EMC (ours)}
	\end{subfigure}
	\begin{subfigure}[t]{0.34\textwidth}
		\centering
		\includegraphics[width=1\linewidth]{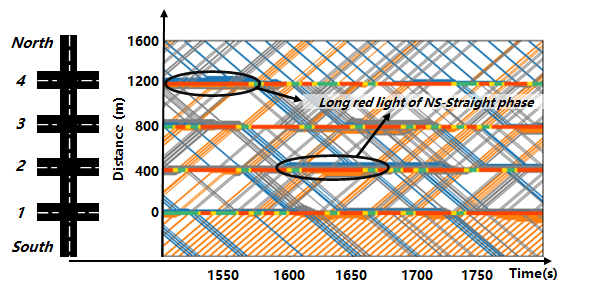}
		\caption{MaxPressure}
	\end{subfigure}
	\begin{subfigure}[t]{0.34\textwidth}
		\centering
		\includegraphics[width=1\linewidth]{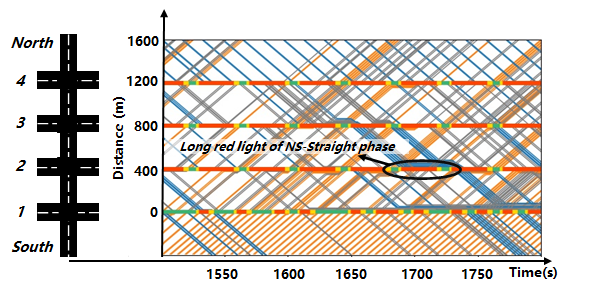}
		\caption{PressLight}
	\end{subfigure}
	\caption{Time-Space Diagram with signal timing plan of our EMC, MaxPressure and Presslight methods.}
	\label{Fig:TSD}
\end{figure*}
\subsection{Experiment Results}
\textbf{Results on Computation Time.}  Table \ref{Tab:avgRes} shows the computation time for each TSC decision phase. From Table \ref{Tab:avgRes}, we can observe traditional methods (i.e., all methods except our EMC) can select to active a movement phase within 0.1 seconds. This can be explained by the fact that 1) for MaxPressure, each intersection only need to select one of the four phases with the maximum pressure, in which the computation time is limited, and 2) for the pre-determined heuristic FixedTime and learned RL methods, each interaction only need to query the corresponding policy by using the real traffic as an input, which is also time-free. The computation time of our EMC method increases with the size of the road network. Fortunately, even in the large-scale with $20\times 20$ intersections, each intersection in EMC can decide to activate a movement phase within 3 seconds. 3 seconds is a reasonable maximum time limit for traffic signal control since there is always a yellow light (e.g., a fixed duration of 3 seconds) after each phase for traffic clean. Therefore, during the yellow light, our EMC can determine next phase in a real-time manner. Moreover, considering our EMC is an anytime algorithm that can search more action space given more time budget and can return a best solution under the existing time, it enables scaling to large-scale TSC problems with hundreds of intersections.

\textbf{Results on Travel Time.} Table \ref{Tab:avgExpTime} shows the average travel time in both real datasets and synthetic datasets, From which we have the following observations:
\begin{itemize}
	\item In synthetic dataset (i.e.,syn\_$4\times 4$, syn\_$15\times 15$ and syn\_$20\times 20$) with medium link and real dataset with long link (i.e., jinan\_$3\times 4$ and hangzhou\_$4\times 4$), compared with heuristic and RL baselines, our EMC method generates the least travel time. Especially in the large-scale synthetic network syn\_$20\time 20$, our EMC reduce ~20\% travel time compared to the best RL method LIT. The potential reason is that for complex large-scale road networks with hundreds of intersections, it is hard for RL methods to train desirable coordination solutions. This result is in consist with the scalability limitations of current multiagent RL methods. In contrast, our EMC method enables direct modeling of the flow influences between neighboring intersections and aims to optimize global traffic pressure. However, MaxPressure only optimizes local pressure of individual intersection, thus the performance decreases in complex large-scale traffic situation.
	\item  In the real datasets ny\_$1\times 16$ and manhattan\_$3 \times 16$ with short link, i.e., the length of the link in the vertical direction is only 100m, the MaxPressure method produces the least travel time. This is can be explained by the fact that a vehicle with the speed of 10m/s can pass through the short link within a decision period (i.e., 10s). Such "pass through" phenomenon will make the prediction-based methods (i.e., RL methods) and online lookahead planning method (i.e., our EMC method) inaccurate in modeling queue state at the next period state. Maxpressure only considers the current period state, which will avoid this "pass through" issue.
	\item  LIT method behaves fairly well in synthetic datasets but does not fit real data with irregular flow data like manhanttan\_$3\times 16$. Presslight shows good result in these datasets except the large-size synthetic datasets. Such divergence conforms to the deficiency in RL methods, that it is incapable of capturing communication between neighboring intersections in large-scale or complex scenarios. The performance advantage of FMA2C over MA2C corresponds to the FMA2C method's optimization effect in serving global coordination.	
\end{itemize}
In summary, for real-world road networks such that no vehicle can pass through an entire road within each period $t$ (i.e., with medium and long links), our EMC is a good option for TSC with respect to real-time responsive and average travel time minimization.

\begin{table}[!t]
	\centering
	\begin{tabular}{c|c|c|c|c}
		\hline\hline
			\multirow{2}*{\backslashbox{Delay (s)}{Dataset}}   & \multirow{2}* {3$\times$3}   & \multirow{2}* {4$\times$ 4} & \multirow{2}* {15$\times$15} & \multirow{2}* {20$\times$20} \\
			 & & & & \\
		\hline\hline
		$\mu$=0 & 0.022 & 0.026 & 0.100 &0.180\\
		  $\mu$=0.01 & 0.616 & 0.629 & 0.719 &0.792\\
		$\mu$=0.02 &1.097 & 1.109 & 1.187 & 1.230\\
		\hline\hline
	\end{tabular}
	\caption{Communication delay (second) of our EMC method.}\label{Tab:CommDelay}
	\vspace{5pt}
\end{table}

\textbf{Results on Communication Delay of our EMC Method.} Since our EMC is message passing-based decentralized method, we also test its communication delay by Pumba \footnote{https://github.com/alexei-led/pumba}, a chaos testing command line tool for Docker containers. In the simulation, we generate 10 Docker nodes, each is allocated with a number of agents/intersections. For example, in a 20$\times$20 road network, 40 agents are randomly selected and allocated to one of the non-configured nodes. Node communicate by the TCP protocol. The communication delay $X$ (ms) follows the Normal distribution $N(\mu,3^2)$ where $\mu$ is the expected delay, and 3 (ms) is the standard deviation. The node ends asynchronously, and the end time of the last node is recorded. Table \ref{Tab:CommDelay} shows the delay of our EMC method at different road networks. From Table \ref{Tab:CommDelay}, we can find that in the large-scale dataset $20\times$20, even with the higher node-to-node communication delay (i.e., 20ms), our EMC takes only 1.2s to communicate messages. This includes that in real-world large-scale 20$\times$20 scenarios, our EMC takes about 3.2s to return the signal decision, in which the 2.0s is used for computing the message (which is shown in Table \ref{Tab:avgRes}) and the remained 1.2s is used for communicating the message.

\subsection{Case Study in Real-World Jinan Road Network}
In this case, we use the time-space diagram to show the trajectory of each vehicle in the real-world road network Jinan\_4$\times$3 in the time range from 1500s to 1800s, and visualize the performance of different methods. In the time-space diagram of Figure \ref{Fig:TSD}, the x-axis is the time and the y-axis is the distance of the westernmost arterial road from south to north. We mainly show traffic flows of north (i.e., blue lines) and south (i.e., orange lines) directions. The flow arrival rate is set as 1.68/s to simulate the peak scenarios. As it is shown in Fig.\ref{Fig:TSD}, there are 4 bands with green-yellow-red colors indicating the signal lights of 4 intersections. We can find that under MaxPressure (i.e.,Fig.\ref{Fig:TSD}(b)) and PressLight (i.e., Fig.\ref{Fig:TSD}(c)), a number of vehicles might accumulate and wait before the red light. As the traffic becomes heavier, MaxPressure and PressLight have to keep the red signals for long periods to make a concession for other directions with high-demand traffic flows. In comparison, most orange and blue lines are straight under our EMC (i.e., Fig.\ref{Fig:TSD}(a)), which indicates few vehicles are waiting before the red light for long periods. This observation further validate our EMC method can detect the actual traffic change on each links and optimize the global pressure in the next period, thereby can feasibly plan short-time green light phase for incoming vehicles from north and south directions to pass through.

\section{Conclusion and Future Work}
This paper studies the TSC problem of minimizing the average travel time of vehicles. To achieve the real-time and  network-level TSC, an explicit multiagent coordination (EMC) framework is proposed, which is simple to implement. In EMC, each intersection is modeled as an autonomous agent, and the coordination efficiency is formulated as the  balance index between neighbor agents. The balance index is general that can be incorporated in a number of other TSC methods. To achieve the real-time responsiveness, an anytime Max-sum\_ADVP-based message passing mechanism is proposed to select the signal plan. The EMC algorithm mainly consists of network-level coordination and local individual action improvement phases, with the aim of minimizing the travel time of vehicles.  The network-level coordination phase is proved to guarantee network stability. Experimental results show that on both synthetic and real datasets, the proposed EMC achieves lower travel time than other benchmark heuristic and RL methods, and can scale to large-scale scenarios.

One interesting future direction is to extend the EMC from one-step lookahead to multi-step lookaheads. This extension is non-trivial since the action space increases exponentially with the number of steps. Decentralized multi-step coordination requires intolerable communication cost (i.e., message size), while centralized coordination requires intolerable computation time. We would like to exploit the interaction structure of EMC (e.g., each agent only interacts with adjacent neighbor agents), and explore the Monte-Carlo tree search (MCTS)-based sampling algorithm \cite{ChoudhuryGMK21} to search coordinated policy that can coordinate within multi-steps.

\section*{Acknowledgment}
This research is supported by the National Natural Science Foundation of China (61806053, 61932007 and 62076060), the Key Research and Development Program of Jiangsu Province of China (BE2022157).

\ifCLASSOPTIONcaptionsoff
  \newpage
\fi



%
\bibliographystyle{IEEEtran}
\bibliography{TITS}



%

\begin{IEEEbiographynophoto}{Wanyuan Wang}
is an associate professor with the School of Computer Science and Engineering, Southeast University, Nanjing, China. He received his PhD. degree in computer science from Southeast University, China, 2016. He has published over 30 articles in refereed journals and conference proceedings, such as the IEEE Transactions, AAAI, and AAMAS. He won the best student paper award from ICTAI14. His main research interests include artificial intelligence, multiagent systems, and optimization.
\end{IEEEbiographynophoto}

\begin{IEEEbiographynophoto}{Tianchi Qiao}
	is currently pursuing the M.E. degree with the School of Computer Science and Engineering, Southeast University. His main research interests include multiagent systems and distributed constraint optimization.
\end{IEEEbiographynophoto}

\begin{IEEEbiographynophoto}{Jinming Ma}
	is currently pursuing the Ph.D. degree with the School of the Computer Science and Technology, University of Science and Technology of China. His research interests include reinforcement learning and multi-agent systems.
\end{IEEEbiographynophoto}

\begin{IEEEbiographynophoto}{Jiahui Jin} 
is an associate professor in the School of Computer Science and Engineering, Southeast University, Nanjing, China. He received his Ph.D. degree in computer science from Southeast University in 2015. He had been a visiting Ph.D. student at University of Massachusetts, Amherst, U.S., during August 2012 to August 2014. His current research interests include large-scale data processing and urban computing.
\end{IEEEbiographynophoto}

\begin{IEEEbiographynophoto}{Zhibin Li} 
received his Ph.D. degree from the School of Transportation at Southeast University, China, in 2014. He is currently a professor at Southeast University. From 2015 to 2017, he worked as a postdoctoral researcher at the University of Washington and the Hong Kong Polytech University. From 2010 to 2012, he was a visiting student at the University of California, Berkeley. His research interests include Intelligent Transportation, Traffic Control, etc. He received China National Scholarship twice, and won the Best Doctoral Dissertation Award by China Intelligent Transportation Systems Association in 2015. He has authored or co-authored over 80 articles in journals. 
\end{IEEEbiographynophoto}

\begin{IEEEbiographynophoto}{Weiwei Wu} 
received the B.Sc. degree from the South China University of Technology and the Ph.D. degree from the Department of Computer Science, City University of Hong Kong (CityU), and University of Science and Technology of
China (USTC) in 2011. He is currently a Professor with the School of Computer Science and Engineering, Southeast University, China. He went to Nanyang Technological University
(NTU, Mathematical Division, Singapore), for
post-doctoral research in 2012. He has published over 50 peer-reviewed papers in international conferences/journals, and serves as TPCs and reviewers for several top international
journals and conferences. His research interests include optimizations
and algorithm analysis, reinforcement learning and game theory.
\end{IEEEbiographynophoto}

\begin{IEEEbiographynophoto}{Yichuan Jiang}
	is a full professor with the School of Computer Science and Engineering, Southeast University, Nanjing, China. He received his PhD degree in computer science from Fudan University, Shanghai, China, in 2005. His main research interests include multiagent systems, social computing, and social networks. He has published more than 90 scientific articles in refereed journals and conference proceedings, such as the IEEE Transactions on Parallel and Distributed Systems, IEEE Journal on Selected Areas in Communications, IEEE Transactions on Systems, Man, and Cybernetics-Part A: Systems and Humans,  IEEE Transactions on Systems, Man, and Cybernetics-Part C: Applications and Reviews, IEEE Transactions on Systems, Man, and Cybernetics: Systems, IEEE Transactions on Cybernetics, the ACM Transactions on Autonomous and Adaptive Systems, the Journal of Autonomous Agents and Multi-Agent Systems, the Journal of Parallel and Distributed Computing, IJCAI, AAAI and AAMAS. He won the best paper award from PRIMA06 and best student paper awards twice from ICTAI13 and ICTAI14. He is a senior member of IEEE.
\end{IEEEbiographynophoto}

\vspace{-20pt}

\end{document}